\def\isarxivversion{1} 
 
\ifdefined\isarxivversion
\documentclass[11pt]{article}
\else
\documentclass[a4paper,UKenglish,cleveref, autoref, thm-restate]{lipics-v2021}
\fi

\usepackage{array}
\usepackage{multirow}
\usepackage{color}
\ifdefined\isarxivversion
\usepackage[english]{babel}
\fi
\usepackage{graphicx}
\usepackage{wrapfig}
\usepackage{epstopdf}
\usepackage{url}
\usepackage{dsfont}
\usepackage{graphicx}
\usepackage{color}
\usepackage{epstopdf}
\usepackage{scrextend}
\usepackage{bbm}
\usepackage{comment}
\usepackage{tikz}

\ifdefined\isarxivversion
\usepackage{hyperref}
\hypersetup{colorlinks=true,citecolor=red,linkcolor=blue} 
\else
\fi
\usetikzlibrary{arrows}

\graphicspath{{./figs/}}

\ifdefined\isarxivversion
\usepackage[margin=1in]{geometry}
\usepackage{algpseudocode}
\usepackage[linesnumbered,ruled,vlined]{algorithm2e}
\else 

\fi

\ifdefined\isarxivversion
\usepackage{amsmath}
\usepackage{amsthm}
\usepackage{amssymb}
\usepackage{subfig}

\newtheorem{theorem}{Theorem}[section]
\newtheorem{lemma}[theorem]{Lemma}
\newtheorem{definition}[theorem]{Definition}
\newtheorem{proposition}[theorem]{Proposition}
\newtheorem{corollary}[theorem]{Corollary}

\newtheorem{remark}[theorem]{Remark}

\newtheorem{claim}[theorem]{Claim}
\else
\usepackage{amsmath}
\usepackage{amsthm}
\usepackage{amssymb}
\usepackage[linesnumbered,lined,boxed,commentsnumbered]{algorithm2e}
\fi

\newtheorem{assumption}[theorem]{Assumption}
\newtheorem{fact}[theorem]{Fact}

\newcommand{\wt}{\widetilde}
\newcommand{\wh}{\widehat}

\newcommand{\supp}{\mathrm{supp}}
\DeclareMathOperator{\poly}{poly}

\DeclareMathOperator{\Z}{{\mathbb Z}}
\DeclareMathOperator{\R}{{\mathbb R}}

\DeclareMathOperator{\F}{{\mathbb F}}
\DeclareMathOperator*{\E}{{\mathbb{E}}}

\definecolor{b2}{RGB}{51,153,255}
\definecolor{mygreen}{RGB}{80,180,0}
\definecolor{yl}{RGB}{255,80,0}

\newcommand{\SSNMF}{{$\mathsf{SSBMF}$}}
\newcommand{\NMF}{{$\mathsf{NMF}$}}
\newcommand{\BMF}{{$\mathsf{BMF}$}}

\newcommand{\Batched}{{$\mathsf{BkV\text{-}SUM}$}}
\newcommand{\ksum}{{$\mathsf{kV\text{-}SUM}$}}

\newcommand{\M}{{\mathbf{M}}}
\newcommand{\W}{{\mathbf{W}}}

\newcommand{\X}{\mathbf{X}}
\DeclareMathOperator{\Id}{Id}

\newcommand{\T}{\mathbf{T}}

\newcommand{\Zhao}[1]{{\color{red}[Zhao: #1]}}

\newcommand{\Ruizhe}[1]{{\color{mygreen}[Ruizhe: #1]}}
\usepackage{thm-restate}

\ifdefined\isarxivversion
\else
\author{Sitan Chen}{UC Berkeley}{sitanc@berkeley.edu}{}{This work was supported in part by NSF Award 2103300, NSF CAREER Award CCF-1453261, NSF Large CCF-1565235 and Ankur Moitra's ONR Young Investigator Award. Part of this work was completed while visiting the Simons Institute for the Theory of Computing}
\author{Zhao Song}{Adobe Research}{zsong@adobe.com}{}{}
\author{Runzhou Tao}{Columbia University}{runzhou.tao@columbia.edu}{}{}
\author{Ruizhe Zhang}{UT Austin}{ruizhe@utexas.edu}{}{This work was supported by NSF Grant CCF-1648712.}
\authorrunning{S. Chen, Z. Song, R. Tao and R. Zhang}
\Copyright{Sitan Chen, Zhao Song, Runzhou Tao and Ruizhe Zhang}
\relatedversion{The full version of this paper is hosted on arXiv: \url{https://arxiv.org/abs/2102.01570}}


\EventEditors{John Q. Open and Joan R. Access}
\EventNoEds{2}
\EventLongTitle{42nd Conference on Very Important Topics (CVIT 2016)}
\EventShortTitle{CVIT 2016}
\EventAcronym{CVIT}
\EventYear{2016}
\EventDate{December 24--27, 2016}
\EventLocation{Little Whinging, United Kingdom}
\EventLogo{}
\SeriesVolume{42}
\ArticleNo{23}

\nolinenumbers
\fi

\begin{document}

\title{Symmetric Sparse Boolean Matrix Factorization and Applications}

\ifdefined\isarxivversion
\author{
Sitan Chen\thanks{\texttt{sitanc@berkeley.edu}. This work was supported in part by NSF Award 2103300, NSF CAREER Award CCF-1453261, NSF Large CCF-1565235 and Ankur Moitra's ONR Young Investigator Award. Part of this work was completed while visiting the Simons Institute for the Theory of Computing} \\ UC Berkeley
\and
Zhao Song\thanks{\texttt{zsong@adobe.com}} \\ Adobe Research
\and
Runzhou Tao\thanks{\texttt{runzhou.tao@columbia.edu}} \\
Columbia University
\and
Ruizhe Zhang\thanks{\texttt{ruizhe@utexas.edu}. This work was supported by NSF Grant CCF-1648712.} \\ UT Austin
}
\else
 

%
\fi

\ifdefined\isarxivversion
\begin{titlepage}
\maketitle
\begin{abstract}
In this work, we study a variant of nonnegative matrix factorization where we wish to find a symmetric factorization of a given input matrix into a sparse, Boolean matrix. Formally speaking, given $\M\in\Z^{m\times m}$, we want to find $\W\in\{0,1\}^{m\times r}$ such that $\| \M - \W \W^\top \|_0$ is minimized among all $\W$ for which each row is $k$-sparse. This question turns out to be closely related to a number of questions like recovering a hypergraph from its line graph, as well as reconstruction attacks for private neural network training.   

As this problem is hard in the worst-case, we study a natural average-case variant that arises in the context of these reconstruction attacks: $\M = \W\W^{\top}$ for $\W$ a random Boolean matrix with $k$-sparse rows, and the goal is to recover $\W$ up to column permutation. Equivalently, this can be thought of as recovering a uniformly random $k$-uniform hypergraph from its line graph.

Our main result is a polynomial-time algorithm for this problem based on bootstrapping higher-order information about $\W$ and then decomposing an appropriate tensor. The key ingredient in our analysis, which may be of independent interest, is to show that such a matrix $\W$ has full column rank with high probability as soon as $m = \widetilde{\Omega}(r)$, which we do using tools from Littlewood-Offord theory and estimates for binary Krawtchouk polynomials.

\end{abstract}
\thispagestyle{empty}
\end{titlepage}

\else

\maketitle
\begin{abstract}

\end{abstract}

\fi



\newcommand{\calS}{\mathcal{S}}
\newcommand{\norm}[1]{\|#1\|}
\newcommand{\U}{\mathbf{U}}
\newcommand{\V}{\mathbf{V}}

\section{Introduction}

Nonnegative Matrix Factorization (\NMF) \cite{agkm12,m13,rsw16,swz17,swz19} is a fundamental problem with a wide range of applications including image segmentation \cite{ls99,sl11}, document clustering \cite{xlg03}, financial analysis \cite{dfdrc08}, music transcription \cite{sb03}, and communication complexity \cite{auy83,n91}. Roughly, given an $n_1\times n_2$ matrix $\M$ and a rank parameter $r > 0$, the goal is to find $n_1\times r$ matrix $\W_1$ and an $n_2\times r$ matrix $\W_2$ for which $\W_1 \W_2^\top$ best approximates $\M$.

As we discuss in Section~\ref{sec:relatedwork}, a number of different variants of {\sf NMF} have been studied in this literature, e.g. 1) constraining the factorization to be \emph{symmetric} in the sense that $\W_1 = \W_2$ \cite{ding2005equivalence,zass2005unifying,yang2012clustering,yan2013learning,chen2008non,kalofolias2012computing,he2011symmetric,wang2011community,kuang2012symmetric, zhang2013overlapping}, 2) constraining the factors to be binary-valued \cite{ban2019ptas,fomin2019approximation,chandran2017parameterized,kumar2019faster,zhang2007binary,ravanbakhsh2016boolean}, and 3) constraining them to be sparse \cite{h04, g12, kp08, so06, gc05, zfrk10}. It turns out that in many situations, it is fruitful to combine all of these desiderata. In this paper, we study the following hybrid of these many variants of \NMF.

Suppose we are given a symmetric nonnegative matrix $\M\in\R^{m\times m}_{\ge 0}$, as well as parameters $k,r\in\mathbb{N}$. Consider the following optimization problem
\begin{equation}
     \min_{\W \in\calS_{m,r,k}  } \norm{\M - \W\W^{\top}}_0 \qquad \text{for} \qquad \calS_{m,r,k}  = \{\W \in \{0,1\}^{m \times r}: \| \W_{j,*} \|_0 = k \ \forall \ j\in[m]\}, \label{eq:factorization2}
\end{equation}
where $\norm{\cdot}_0$ denotes the number of nonzero entries, where matrix multiplication is either over the reals or over the Boolean semiring\footnote{In the Boolean semiring, addition is given by logical OR, and multiplication is given by logical AND}. We will refer to this problem as \SSNMF{} (sparse symmetric Boolean matrix factorization) in the sequel. 

The particular set $\calS_{m,r,k}$ we optimize over turns out to have a number of natural motivations in graph clustering, combinatorics of hypergraphs, and more recently, ML security.

\subparagraph*{Clique Decomposition} For instance, consider the problem of identifying community structure in a social network. In the real world, there will be overlaps between communities, and a natural goal might be to identify a collection of them that covers the graph but such that every node only occurs in a limited number of communities. In \eqref{eq:factorization2}, we can think of $\M$ as the adjacency matrix of some graph $G$. Now note that over the Boolean semiring, if we had $\M = \W\W^{\top}$, then $\W$ would encode a \emph{clique cover} of $G$ where every vertex belongs to a small number of cliques. Indeed, we can think of the parameter $r$ as the total number of cliques in the cover and $k$ as the number of cliques to which any one of the $m$ vertices belongs, in which case the nonzero entries in the $j$-th column of $\W$ indicate which vertices belong to the $j$-th clique.

\paragraph{Line Graphs of Hypergraphs} In combinatorics, there is a large body of work on the problem of recovering hypergraphs from their \emph{line graphs} \cite{rou73, l74, lov77, s82, w92, ds95, jkl97, metelsky1997line,skums2005edge,ltv15}, which turns out to be equivalent to the version of \eqref{eq:factorization2} when the input matrix $\M$ admits an exact factorization. We can regard any $\W\in\calS_{m,r,k}$ as the incidence matrix of a $k$-uniform hypergraph $H$ with $m$ hyperedges and $r$ vertices, so if we work over the Boolean semiring, $\M \triangleq \W\W^{\top}$ is exactly the adjacency matrix for the line graph of $H$, namely the graph whose vertices correspond to hyperedges of $H$ and whose edges correspond to pairs of hyperedges which overlap on at least one vertex of $H$. 

When $k = 2$, Whitney's isomorphism theorem \cite{w92} characterizes which graphs are uniquely identified by their line graphs; in our notation, this theorem characterizes when one can uniquely identify $\W$ (up to permutation) from $\M = \W\W^{\top}$. In this case, \cite{rou73, l74, s82,ds95,ltv15} have given efficient algorithms for recovering $\W$ from $\M$. Unfortunately, for $k > 2$, no analogue of Whitney's theorem exists~\cite{lov77}. In fact, even determining whether a given graph is the line graph of some $k$-uniform hypergraph is $\mathsf{NP}$-complete \cite{prt81}. Rephrased in the language of \eqref{eq:factorization2}, this implies the following worst-case hardness result for sparse symmetric Boolean NMF:
\begin{theorem}[\cite{prt81}]\label{thm:hard}
    If $\mathsf{P}\neq\mathsf{NP}$, no polynomial-time algorithm can take an arbitrary matrix $\M$ and decide whether \eqref{eq:factorization2} has zero objective value.
\end{theorem}




A number of results for reconstructing a hypergraph $H$ from its line graph are known under additional assumptions on $H$ \cite{jkl97,metelsky1997line,skums2005edge}. In a similar spirit, in this work we ask whether there are natural \emph{average-case} settings where one could hope to do this.. One such setting is suggested by the following recent application of \SSNMF{} to ML security.

\paragraph{Attacks on InstaHide} Interestingly, \SSNMF{} has arisen implicitly in a number of works \cite{clsz21,carlini_attack,hstzz20} attacking InstaHide, a recently proposed scheme for privately training neural networks on image data \cite{huang2020instahide}. At a high level, the premise for InstaHide was to train on a synthetic dataset essentially consisting of random linear combinations of $k$ images from the private dataset \cite{zhang2017mixup}, rather than on the private dataset itself.

As we discuss in Section~\ref{sec:instahide_defer}, attacking this scheme amounts to a certain natural variant of the classic $k$-sum problem, and the aforementioned attacks reduce from solving this problem to solving an instance of \SSNMF{}. Roughly speaking, they show how to form the Gram matrix $\M$ whose rows and columns are indexed by images in the synthetic dataset such that the $(i,j)$-th entry of $\M$ is 1 if the sets of private images that give rise to synthetic images $i$ and $j$ overlap, and 0 otherwise. Similar to the clique cover example above, we can then think of the optimal $\W$ in \eqref{eq:factorization2} as encoding which private images were used to generate the synthetic images. 

We emphasize that the attacks on InstaHide \cite{clsz21,carlini_attack,hstzz20} were able to devise efficient algorithms for \SSNMF{} precisely because the instances that arose there were average-case: concretely, the rows of the optimal $\W$ were random $k$-sparse binary vectors.

That said, the key drawback of these algorithms is that they either lack provable guarantees \cite{carlini_attack}, require extremely large $m$ \cite{clsz21}, or only apply to $k=2$ \cite{hstzz20}.

\subsection{Our Results}

Motivated by the average-case version of \SSNMF{} that arises in the above security application, as well as the shortcomings in the aforementioned works in this setting, in this work we focus on the following distributional assumption:

\begin{assumption}\label{assume:random}
    Every row of $\W\in\{0,1\}^{m\times r}$ is an independent, uniformly random $k$-sparse bitstring, and the algorithm takes as input the matrix $\M = \W\W^{\top}$ over the Boolean semiring.
\end{assumption}

\begin{remark}
    Note that an algorithm that works under Assumption~\ref{assume:random} can easily be modified to handle the setting where $\M = \W\W^{\top}$ over the integers/reals instead of the Boolean semiring. The reason is that the matrix $\M'$ whose $(i,j)$-th entry is $\mathds{1}{\M_{ij} > 0}$ satisfies $\M' = \W\W^{\top}$ over the Boolean semiring. So in the ``realizable'' setting of Assumption~\ref{assume:random}, working over the Boolean semiring is more general than working over the integers/reals.
\end{remark}

Our main result is to give a polynomial-time algorithm for \eqref{eq:factorization2} in this setting:

\begin{theorem}[Average-case guarantee, informal version of Theorem~\ref{thm:avgcase}]\label{thm:avgcase_intro}
    Fix any integer $2\le k \le r$, failure probability $\delta \in (0,1)$, and suppose $m \ge \widetilde{\Omega}(rk\log (1/\delta) )$. Suppose $\W,\M$ satisfy Assumption~\ref{assume:random}, where matrix multiplication is over the Boolean semiring. There is an algorithm which takes as input $\M$, runs in $O(m^{\omega + 1})$ time\footnote{$\omega \approx 2.373$ is the exponent of matrix multiplication.} and, with probability $1 - \delta$ over the randomness of $\W$, outputs a matrix $\widehat{ \W } \in \{ 0 , 1 \}^{ m \times r }$ whose columns are a permutation of those of $\W$.
\end{theorem}

We stress that the minimum $m$ for which Theorem~\ref{thm:avgcase_intro} holds is a fixed polynomial in $r,k$; in contrast, the only known provable result to work in this setting for general $k$ \cite{clsz21} used a rather involved combinatorial argument to ``partially'' factorize $\M$ when $m = \Omega(n^{k-2/k})$.

Furthermore, we emphasize that our dependence on $r$ is \emph{nearly optimal}:

\begin{remark}\label{remark:nearoptimal}
    The connection to line graphs makes clear why $m$ must be at least $\Omega(r\log r)$ for unique recovery of $\W$ (up to permutation) from $\M = \W\W^{\top}$ to be possible, even for $k = 2$. In this case, $\M$ is simply the adjacency matrix for the line graph of a random (multi)graph $G$ given by sampling $m$ edges with replacement. It is well known that such a graph is w.h.p. not connected if $m = o(r\log r)$ \cite{erdHos1960evolution}, so by Whitney's theorem there will be multiple non-isomorphic graphs for which $\M$ is the adjacency matrix of their line graph.
\end{remark}

We remark that in the language of line graphs, Theorem~\ref{thm:avgcase_intro} says that \emph{whereas it is $\mathsf{NP}$-complete to even recognize whether a given graph is a line graph of some hypergraph in the worst case, reconstructing a \emph{random} hypergraph from its line graph is tractable.}

To prove Theorem~\ref{thm:avgcase_intro}, we design an algorithm that first bootstraps the ``pairwise'' information present in $\M$ into third-order information in the form of the tensor $\T = \sum^r_{i=1}\W_i^{\otimes 3}$, where $\W_i$ is the $i$-th column of $\W$ (see Section~\ref{subsec:formtensor} of the technical overview for a summary of how $\T$ is constructed from $\M$). Once $\T$ has been constructed, we would like to invoke standard algorithms for tensor decomposition to recover the $\W_i$'s.

However, the main technical hurdle we must overcome before we can apply tensor decomposition is to prove that $\W$ is full column rank with high probability (see Theorem~\ref{thm:independent_of_W_intro}). As we discuss in Section~\ref{sec:lindep_overview} of the overview, this can be done with a straightforward net argument if $m$ scales at least quadratically in $r$. As we are interested in getting an optimal dependence on $m$ in terms of $r$ however, the bulk of the technical content in this paper goes into showing this holds even when $m$ scales near-linearly in $r$. To achieve this, our analysis appeals to an array of technical tools from discrete probability, e.g. estimates for binary Krawtchouk polynomials, a group-theoretic Littlewood-Offord inequality, and modern tools \cite{fks20} for showing nonsingularity of random \emph{square} Boolean matrices.

\paragraph{Connection to the $k$-sum problem}

As we alluded to above, the aforementioned attacks on InstaHide elucidated a simple connection between \SSNMF{} and the following natural variant of $k$-sum that we call \emph{batched-$k$-vector sum}, or \Batched{} for short, for which Theorem~\ref{thm:avgcase_intro} can also be leveraged to obtain average-case guarantees.

Suppose there is a database $\X$ which is a list of $d$-dimensional vectors $x_1, \cdots, x_r$. For a fixed integer $k$, we are given vectors $y_1, \cdots, y_m$, where for each $j\in[m]$, we promise that there is a set $S_j$ with size $k$ for which $y_j = \sum_{i \in S_j} x_i$. Given $y_1, \cdots, y_m$, our goal is to recover sets $S_1, \cdots, S_m$, even if $\X$ is unknown to us. We refer to Definition~\ref{def:bksum} for a formal definition of \Batched{}.

\cite{clsz21,carlini_attack} show that when $x_i$ are Gaussian vectors or real images, it is possible to construct a \emph{similarity oracle} that, given any $y_i, y_j$, returns the number of $x$'s they share (i.e., $|S_i \cap S_j|$). In the presence of such an oracle, it is easy to see that there is a reduction from \Batched{} to \SSNMF{}.

For reconstruction attacks however, we would like even finer-grained information, e.g. the actual entries of the vectors $x_i$ used to generate the $y_i$'s. We give an algorithm for this that, given the information obtained by \SSNMF{}, recovers all of the ``heavy'' coordinates of the $x_i$'s.
\begin{theorem}[Informal version of Theorem~\ref{thm:instahide_formal}]\label{thm:instahide_informal}
    Fix any $k \in\mathbb{N}$ and failure probability $\delta \in (0,1)$, and suppose $m \ge \widetilde{\Omega}(rk\log (d/\delta ) )$. Given a synthetic dataset of $m$ $d$-dimensional vectors generated by batched $k$-vector sum, together with its similarity oracle, there is an $O(m^{\omega + 1} + d\cdot r\cdot m)$-time algorithm for approximately recovering the magnitudes of the ``heavy'' coordinates of every vector in the original dataset $\X = \{x_1,\ldots,x_r\}$. Here, a coordinate of a vector is ``heavy'' if its magnitude is $\Omega(k)$ times the average value of any original vector in that coordinate.
\end{theorem}

In fact, this theorem applies even if we only get access to the \emph{entrywise absolute values} of the $y_i$'s, yielding the first provable attack on InstaHide which is polynomial in all parameters $m,d,k,r$ (see Section~\ref{sec:defer_recover_heavy} for details).

\paragraph{Worst-Case Guarantee}

Finally, we observe for the worst-case version of \eqref{eq:factorization2}, it is quite straightforward to get a quasipolynomial-time approximation algorithm by setting up an appropriate CSP and applying known solvers \cite{dm18}:

\begin{theorem}[Worst-case guarantee, informal version of Theorem~\ref{thm:qptas_bmf_sym_app}]\label{thm:csp_informal}
Given a symmetric matrix $\M \in \Z^{m \times m}$, rank parameter $r\le m$, and accuracy parameter $\epsilon \in (0, 1] $, there is an algorithm that runs in $m^{ O ( \epsilon^{-1} k^2 \log r ) }$ time and outputs $\widehat{\W} \in \calS_{m,r,k}$ such that
\begin{align*}
\| \M - \widehat{\W} \widehat{\W}^\top\|_0 \le \min_{\W\in\calS_{m,r,k}}\norm{\M - \W\W^{\top}}_0 + \epsilon m^2,
\end{align*}
where matrix multiplication can be over $\Z$ or over the Boolean semiring.
\end{theorem}

\subsection{Related Work}
\label{sec:relatedwork}

\paragraph{Nonsingularity of Boolean Matrices} The nonsingularity of random matrices with binary-valued entries has been well-studied in random matrix theory, especially in the setting where each entry is an i.i.d. random draw from $\{\pm 1\}$ (see for example, \cite{cv07,rv08,cv08}). However, in our setting, there exists some dependence within a row since we require that the row sum equals to $k$. While there have been a number of works under this setting (for example \cite{n13,fjls19,jv19, ap20,fks20}), they all studied the case when $k$ is large (i.e., when $k=\Omega(\log r)$). 

For $k = O(1)$, there was a long-standing conjecture from \cite{vu08} that the adjacency matrix of a random $k$-regular graph is singular with probability $o(1)$ when $3 \leq k < r$. A recent work \cite{hua18} fully resolved this conjecture using a local CLT and large deviation estimate. In contrast, our $\W$ corresponds to the adjacency matrix of a bipartite graph with only \textit{left-regularity}. Furthermore, our proof uses only elementary techniques, and our emphasis is on showing that for somewhat tall rectangular matrices, the columns are linearly dependent with probability $1/\poly(r)$ for any constant $k\geq 1$.

\paragraph{Fountain Codes} Another line of work that studies the nonsingularity of random matrices with discrete-valued entries is that of \emph{fountain codes} \cite{mackay2005fountain,luby2002lt}. While an exposition of this literature would take us too far afield, at a high level many of these works are interested in establishing that for an $m\times r$ matrix whose rows are independently sampled from a particular distribution over $\mathbb{F}_q^r$ (ideally supported over sparse vectors to allow for fast decoding), the so-called \emph{MDS} property holds. This is the property that any $r\times r$ submatrix of the matrix is full rank. Note that this is an even stronger property than linear independence of the columns. While these works have considered distributions somewhat similar to ours (e.g. \cite{asteris2014repairable} consider a distribution over vectors of sparsity \emph{at most} $k$ where one samples $k$ coordinates from $[r]$ with replacement, assigns those entries to be random from $\mathbb{F}_q$, and sets all other entries to be zero), to our knowledge they all require the sparsity of the rows to be at least logarithmic in $r$.

\paragraph{Comparison to \cite{adm+18}} Perhaps the work most closely related to ours is that of \cite{adm+18}. In our notation, they consider the following problem. There is an unknown matrix $\W\in\R^{m\times r}$, and for a fixed parameter $\ell\in\mathbb{N}$, they would like to recover $\W$ given the $\ell$-th order tensor $\sum_i \W_i^{\otimes \ell}$. They show that when the columns of $\W$ are sampled from a sufficiently anti-concentrated distribution and $r\le m^{\lfloor\frac{\ell-1}{2}\rfloor}$, then they can recover the $\W_i$'s from the tensor.

While this appears on the surface to be quite related both to \SSNMF{} and to our choice of algorithm, we emphasize a crucial distinction. In \SSNMF{} we work with $\ell = 2$ as we are simply given the matrix $\W\W^{\top} = \sum_i \W_i^{\otimes 2}$, and for this choice of $\ell$, the guarantees of \cite{adm+18} are vacuous as $\lfloor \frac{\ell - 1}{2}\rfloor = 0$. There is a good reason for this: their algorithm is based on directly applying tensor decomposition to the tensor $\sum_i \W_i^{\otimes \ell}$ that is given as input. In contrast, a key innovation in our work is to ``bootstrap'' a higher-order tensor out of only the lower-order information given by $\M = \W\W^{\top}$.

Of course, the other difference between \cite{adm+18} and our work is that they work in a smoothed analysis setting, while we work in a random setting. As a result, even though a key step in both our analysis and theirs is to show that under our respective distributional assumptions, the columns of $\W$ are full rank with high probability, our techniques for doing so are very different. We also remark that while the smoothed analysis setting is qualitatively more flexible, their analysis suffers the same drawback as the abovementioned works on nonsingularity of Boolean matrices and fountain codes, namely it cannot handle the setting of Assumption~\ref{assume:random} if the rows of $\W$ are $o(\log r)$-sparse.

\paragraph{Standard {\NMF}}
Lastly, we give an overview of the large literature on nonnegative matrix factorization and its many variants. The most standard setting of {\NMF} is $\min_{\U,\V\ge 0} \norm{\M - \W_1\W_2^{\top}}$, where $\W_1,\W_2$ range over all possible $m\times r$ matrices with nonnegative entries. This has been the subject of a significant body of theoretical and applied work, and we refer to the survey \cite{g12} for a comprehensive overview of this literature.

\paragraph{Symmetric {\NMF}}
When we constrain $\W_1$ and $\W_2$ to be equal, we obtain the question of \emph{symmetric {\NMF}}, which is closely related to kernel $k$-means \cite{ding2005equivalence} and has been studied in a variety of contexts like graph clustering, topic modeling, and computer vision \cite{zass2005unifying,yang2012clustering,yan2013learning,chen2008non,kalofolias2012computing,wang2011community,zhang2013overlapping}. Symmetric {\NMF} has received comparatively less attention but is nevertheless a popular clustering technique \cite{he2011symmetric,kuang2012symmetric,ding2005equivalence} where one takes the input matrix $\M$ to be some similarity matrix for a collection of data points and interprets the factorization $\W$ as specifying a soft clustering of these points into $r$ groups, where $\W_{i,\ell}$ is the ``probability'' that point $i$ lies in cluster $\ell$. 

While there exist efficient provable algorithms for asymmetric {\NMF} under certain separability assumptions~\cite{agkm12,m13}, the bulk of the work on symmetric {\NMF} has been focused on designing iterative solvers for converging to a stationary point~\cite{he2011symmetric,kuang2012symmetric,vandaele2016efficient,lu2017nonconvex,zhu2018dropping}; we refer to the recent work of \cite{dragomir2019fast} for one such result and an exposition of this line of work.

\paragraph{Binary Matrix Factorization}

In \NMF{}, when we constrain $\W_1,\W_2$ to be matrices with binary-valued entries, the problem becomes \emph{binary matrix factorization} \cite{zhang2007binary,chandran2017parameterized,ban2019ptas,fomin2019approximation,kumar2019faster,kueng2021binary1,kueng2019binary2}, which is connected to a diverse array of problems like LDPC codes \cite{ravanbakhsh2016boolean}, optimizing passive OLED displays \cite{kumar2019faster}, and graph partitioning \cite{chandran2017parameterized}.

With the exception of \cite{zhang2013overlapping,kueng2021binary1} which considered community detection with overlapping communities, most works on binary matrix factorization focus on the asymmetric setting. Over the reals, this is directly related to the bipartite clique partition problem \cite{orlin1977contentment,chalermsook2014nearly,chandran2017parameterized}. \cite{kumar2019faster} gave the first constant-factor approximation algorithm for this problem that runs in time $2^{O(r^2\log r)}\poly(m)$. Our Theorem~\ref{thm:csp_informal} also extends to this asymmetric setting (see Theorem~\ref{thm:qptas_bmf}); see Remark~\ref{remark:compare_kumar} for a comparison.

On the other hand, over the Boolean semiring, this problem is directly related to the \emph{bipartite clique cover} problem. Also called \emph{Boolean factor analysis}, the \emph{discrete basis problem}, or \emph{minimal noise role mining}, it has received substantial attention in the context of topic modeling, database tiling, association rule mining, etc. \cite{seppanen2003simple,vsingliar2006noisy,belohlavek2010discovery,miettinen2008discrete,vaidya2007role,lu2012constraint,mitra2016survey}. The best algorithm in this case, due to \cite{fomin2019approximation}, runs in time $2^{2^{O(r)}}\cdot m^2$, matching the lower bound of \cite{chandran2017parameterized}. \cite{eu18} considered the decision version of this problem, whose goal is to decide the minimum clique cover size $r$, and proved that it is $\mathsf{NP}$-hard if $k\geq 5$. A more general version of this problem was studied by \cite{agss12} with applications to community detection. Our worst-case algorithm also applies to this setting (see Corollary~\ref{cor:semiring}). In particular, even without the sparsity condition, the running time still matches the lower bound (see Remark~\ref{rmk:compare_bool}).

Over the reals, \cite{kueng2021binary1} gives sufficient conditions on $\W\in\{0,1\}^{m\times r}$ for which one can recover the matrix $\W\W^{\top}$, namely that the set of Hadamard products between columns of $\W$ spans a $(\binom{r}{2}+1)$-dimensional space. In this case, they give an SDP-based algorithm for recovering $\W$. While it is conceivable that for sufficiently large $m$, $\W$ also satisfy this property with high probability under our Assumption~\ref{assume:random}, in our setting an even simpler condition suffices, namely that $\W$ has full column rank. As we will see, even this turns out to be quite nontrivial to show when $m$ is small.

More generally, we refer to the comprehensive survey of \cite{miettinen2020recent} for other results on {\BMF}.

\paragraph{Sparse \NMF{}}

In sparse {\NMF}, one enforces that the factors $\W_1, \W_2$ must be sparse matrices~\cite{h04, g12} as this can lead to more interpretable results, e.g. in speech separation~\cite{so06}, cancer diagnosis~\cite{gc05} and facial expression recognition~\cite{zfrk10}. Hoyer \cite{h04} initiated the study of this problem based on the observation that standard {\NMF} usually outputs a sparse solution and modeled Sparse {\NMF} by adding a sparsity ($L_1$-norm divided by $L_2$-norm) penalty term to the minimization. Gillis \cite{g12} studied this problem from a geometric perspective via data preprocessing. Apart from this regularization approach, there are also works that deal with exact sparsity constraints \cite{pp12, cg19}.


\subsection{Miscellaneous Notation}

We use $\mathbb{F}_q$ to denote finite field and sometimes use $\mathbb{F}$ to denote $\mathbb{F}_2$. For any nonnegative function $f$, we use $\wt{O}(f)$ to denote $f \poly(\log f)$ and use $\wt{\Omega}(f)$ to denote $f / \poly(\log f)$.

For a vector $x \in \R^n$, we use $\supp(x)$ to denote the nonzero indices, $\| x \|_0$ to denote the number of nonzero entries, and $\| x \|_p$ to denote its $\ell_p$ norm. We use $\vec{1}_d$ to denote the all-ones vector in $\R^d$ and suppress the subscript when the context is clear.

For a matrix $A$, we use $A^+$ to denote the pseudo-inverse of matrix $A$. We use $\| A \|_F$ to denote the Frobenius norm of matrix $A$, $\| A \|_1 = \sum_{i,j} |A_{i,j}|$ to denote its entrywise $\ell_1$ norm, $\| A \|_0$ to denote the number of nonzero entries, and $\| A \|_{\infty}$ to denote $\max_{i,j} |A_{i,j}|$.

Given $r,k$, we let $\binom{r}{[k]}$ denote the collection of all subsets of $[r]$ of size $k$.

\section{Technical Overview}

As the technical details for Theorems~\ref{thm:instahide_informal} and \ref{thm:csp_informal} are more self-contained, in this section we focus on highlighting the key technical ingredients for our main result, Theorem~\ref{thm:avgcase_intro}.

The starting point is the following thought experiment. If instead of getting access to the matrix $\M = \sum^r_{i=1}\W_i^{\otimes 2}$ over the Boolean semiring, suppose we had the \emph{tensor} 
\begin{equation}
    \T = \sum^r_{i=1}\W_i^{\otimes 3} \label{eq:Tdef}
\end{equation} over $\Z$. Provided the columns $\W_i$ of $\W$ are linearly independent over $\R$, then we can run a standard tensor decomposition algorithm to recover $\W_1,\ldots,\W_r$ up to permutation. As such, there are two technical steps: 
\begin{itemize}
    \item Step 1. bootstrapping the tensor $\T$ given only $\M$,
    \item Step 2. showing linear independence of the $\W_i$.
\end{itemize}

\subsection{Bootstrapping the Tensor}
\label{subsec:formtensor}

The key insight is that although the similarity matrix $\M$ only gives access to ``second-order'' information about correlations between the entries of $\W$, we can bootstrap third-order information in the form of $
    \T \triangleq \sum^r_{i=1}\W_i^{\otimes 3},
$ where $\W_i$ is the $i$-th column of $\W$ and the tensor is defined over $\R$ rather than the Boolean semiring.

Given sets $S_1,\ldots,S_c\subset[r]$, let $\mu(S_1,\ldots,S_c)$ denote the probability that a randomly chosen subset $T$ of $[r]$ of size $k$ does not intersect any of them. It is easy to see that \begin{equation}
    \mu(S_1,\ldots,S_c) = \binom{r - |S_1\cup\cdots\cup S_c|}{k}\bigg /\binom{r}{k} \triangleq \mu_{|S_1\cup\cdots \cup S_c|}. \label{eq:mudef}
\end{equation}
The following fact implies that $\mu$ is monotonically decreasing in $|S_1\cup\cdots\cup S_c|$ and, importantly, that the different $\mu_t$'s are well-separated.
\begin{fact}[Informal version of Fact~\ref{fact:gap}]\label{fact:gap_intro}
    There exist absolute constants $C, C' > 0$ for which the following holds. If $r \ge C\cdot k^2$, then for any $0 \le t\le 3k$, we have that $\mu_t \ge \mu_{t+1} + C' k/r$ and $\mu_t \ge 1 - O(tk/r)$.
\end{fact}
Hence, we first estimate $\mu$ by computing the fraction of columns of $\M$ which are simultaneously zero in rows corresponding to the sets $S_1,\dots,S_c$. Provided that the number of columns $m$ is sufficiently large, we can estimate this probability to within error $O(k/r)$, and then invert along $\mu$ to exactly recover $|S_1\cup \cdots\cup S_c|$, from which we obtain $\T_{a,b,c}$ via:
\begin{align*}
\T_{a,b,c} = |S_1\cap S_2\cap S_3| = |S_1 \cup S_2\cup S_3| - |S_1\cup S_2| - |S_2\cup S_3| - |S_1\cup S_3| + 3k,    
\end{align*}
for any $a,b,c\in[m]$ with the corresponding subsets $S_1,S_2,S_3$.
\begin{lemma}[Informal version of Lemma~\ref{lem:extractT}]\label{lem:extractT_intro}
    If $m \ge \wt{\Omega}\left(rk\log(1/\delta)\right)$, then with probability at least $1 - \delta$ over the randomness of $\M$, there is an algorithm (Algorithm~\ref{alg:tensor_recover_intro})
    for computing $\T_{a,b,c}$ for any $a,b,c\in[m]$ in time $O(m^{\omega + 1})$.
\end{lemma}
We defer a full proof of Lemma~\ref{lem:extractT_intro} to Section~\ref{app:defer_lem_extractT}.
As for the runtime, we can compute each slice of $\T$ by setting up an appropriate matrix multiplication. Then, it suffices to apply the following standard guarantee for (noiseless) tensor decomposition:
\begin{lemma}[Jennrich's algorithm, see e.g. Lemma~\ref{lem:jennrich}]\label{lem:jennrich_intro}
    Given a collection of linearly independent vectors $w_1,\ldots,w_r\in\R^m$, there is an algorithm that takes any tensor $\T = \sum^r_{i=1}w_i^{\otimes 3}$, runs in time $O(m^{\omega})$, and outputs a list of vectors $\wh{w}_1,\ldots,\wh{w}_r$ for which there exists permutation $\pi$ satisfying $\wh{w}_i = w_{\pi(i)}$ for all $i\in[r]$.
\end{lemma}

The full algorithm for recovering $\W$ from $\M$ is given in Algorithm~\ref{alg:tensor_recover_intro} below.

\begin{algorithm}[!ht]
\caption{\textsc{TensorRecover}($\M$)}\label{alg:tensor_recover_intro}
	\DontPrintSemicolon
	\LinesNumbered
	\KwIn{$\M\in\{0,1\}^{m\times m}$ s.t. $\M = \W\W^{\top}$ over Boolean semiring for some $\W\in\calS_{m,r,k}$}
	\KwOut{Matrix $\wh{W}\in\calS_{m,r,k}$ which is equal to $\W$ up to column permutation}
	    \tcc{Form the tensor $\T$} 
	    \For{$(a,b,c)\in[m]\times[m]\times[m]$} {
	        Let $\mu_{abc}$ be the fraction of $\ell\in[m]$ for which $\M_{a,\ell}=\M_{b,\ell}=\M_{c,\ell}$ are all zero. Define $\mu_{ab},\mu_{ac},\mu_{bc}$ analogously.\tcp*{ (Lemma~\ref{lem:extractT_intro})} 
	        Let $t_{abc}$ be the nonnegative integer $t$ for which $\mu_{t}$ (see \eqref{eq:mudef}) is closest to $\mu_{abc}$. Define $t_{ab},t_{ac},t_{bc}$ analogously.\;
	        $\T_{a,b,c}\gets t_{abc} - t_{ab} - t_{ac} - t_{bc} + 3k$.\;
	    }
	    \tcc{Run Jennrich's on $\T$}
	    Randomly sample unit vectors $v_1,v_2\in\mathbb{S}^{m-1}$.\;
	    $\M_1 \gets \T(\Id,\Id,v_1)$, $\M_2 \gets \T(\Id,\Id,v_2)$.\;
	    Let $\wt{w}_1,\ldots,\wt{w}_r\in\R^m$ denote the left eigenvectors outside the kernel of $\M_1\M_2^+$.\;
	    Round $\{\wt{w}_i\}$ to Boolean vectors $\{v_i\}$; let $\wh{\W}\in\{0,1\}^{m\times r}$ consist of $\{w_i\}$.\;
	    \Return $\wh{\W}$.\;
\end{algorithm}

\subsection{Linear Independence}
\label{sec:lindep_overview}

To use Jennrich's, it remains to show that the columns of $\W$ are indeed linearly independent with high probability. This is the technical heart of our analysis. We show:
\begin{theorem}[Linear independence of $\W$, informal version of Theorem~\ref{thm:independent_of_W}]\label{thm:independent_of_W_intro}
Let $\W\in \{0,1\}^{m\times r}$ be a random matrix whose rows are i.i.d. random vectors each following a uniform distribution over $\{0,1\}^r$ with exactly $k$ ones. For constant $k\geq 1$ and $m= \Omega(\max(r,(r/k)\log r))$, the $r$ columns of $\W$ are linearly independent in $\R$ with probability at least $1-\frac{1}{\poly(r)}$.
\end{theorem}
Note that by a simple net argument, when $m = \wt{\Omega}(r^2)$ one can show that $\W$ is not only full rank, but polynomially well-conditioned. But as our emphasis is on having the sample complexity $m$ depend near-optimally on the rank parameter $r$, we need to be much more careful.

When $k$ is odd, it turns out to be more straightforward to prove Theorem~\ref{thm:independent_of_W_intro}. In this case, to show linear independence of the $\W_i$'s over $\R$, we first observe that it suffices to show linear independence over $\F_2$. For a given $u\in\F^r_2$, one can explicitly compute the probability that $\W u = 0$, and by giving fine enough estimates for these probabilities based on bounds for binary Krawtchouk polynomials and taking a union bound, we conclude that the columns of $W$ are linearly independent with high probability in this case.

This proof strategy breaks down for even $k$ because in this case $\W$ is not full-rank over $\F_2$: the columns of $\W$ add up to zero over $\F_2$. Instead, we build on ideas from \cite{fks20}, which
studies the square matrix version of this problem and upper bounds the probability that there exists some $x$ for which $\W x = 0$ and for which the most frequent entry in $x$ occurs a fixed number of times. Intuitively, if the most frequent entry does not occur too many times, then the probability that $\W x=0$ is very small. Otherwise, even if the probability is large, there are ``few'' such vectors. Unfortunately, \cite{fks20} requires $k\geq \Omega(\log r)$, and in order to handle the practically relevant regime of $k = O(1)$, we need to adapt their techniques and exploit the fact that $\W$ is a slightly tall matrix in our setting. 

In particular, we leverage a certain group-theoretic Littlewood-Offord inequality \cite{js17} which may be of independent interest to the TCS community.  More specifically, we consider two scenarios under which there is a vector $x$ such that $\W x = 0$. The first one is that there is a vector $x$ with $\W x = 0$ which has some zero entries, in which case we can actually still reduce to the $\F_2$ case as in the case of odd $k$. The second one is that there is a vector $x$ with $\W x = 0$ where all entries of $x$ are non-zero. In this case, since $\mathbf{1}\in \ker(\W)$ in $\F_2$, we cannot consider linear independence over $\F_2$, and our main workaround is instead to work over the cyclic group $\mathbb{Z}_q$ where $q\in [k-1, 2k]$ is a prime. 

For this latter case, we further divide into two more cases: (1) the most frequent entry in $x$ occurs a large number of times; (2) $x$ does not have such an entry. By a union bound over both kinds of vectors $x$, it suffices to show anti-concentration in the sense that we want to upper bound $\Pr_w[\langle w, x\rangle=a\pmod{q}]$ for a fixed $x$ in either case. In case (1), we can use the upper bound from \cite{fks20}. In case (2), \cite{fks20} used a Littlewood-Offord-type inequality \cite{e45} to upper bound this probability, which does not apply in our case because it works for real numbers and we cannot apply union bound for an infinite number of vectors. Instead, we use a group-theoretic Littlewood-Offord inequality \cite{js17} and follow the idea of \cite{fks20} to transform the uniform $k$-sparse distribution to an i.i.d. Bernoulli distribution, completing the proof.

We defer the details of Theorem~\ref{thm:independent_of_W_intro} to Section~\ref{app:main_indep}. Altogether, this allows us to conclude Theorem~\ref{thm:avgcase_intro}, and the formal proof is in Section~\ref{subsec:puttogether_tensor}.

\paragraph{Roadmap}

In Section~\ref{sec:preli_app} we provide technical preliminaries, basic definitions, and tools from previous work. In Section~\ref{app:defer_tensor} we prove our average-case guarantee, Theorem~\ref{thm:avgcase_intro}. In Section~\ref{sec:instahide_defer} we describe the connection between \SSNMF{}, \Batched{}, and private neural network training in greater detail and then prove Theorem~\ref{thm:instahide_informal}. In Appendix~\ref{app:defer_csp} we prove our worst-case guarantee, Theorem~\ref{thm:csp_informal}.

\section{Preliminaries}
\label{sec:preli_app}

\subsection{Notations}
\label{sec:notations}
We introduce some notations and definitions we will use throughout this paper.

For a positive integer $n$, we use $[n]$ to denote the set $\{1,2,\cdots, n\}.$
For a vector $x$, we use $\| x \|_2$ to denote its $\ell_2$ norm. We use $\| x \|_1$ to denote its $\ell_1$ norm.

For a matrix $A$, we use $\| A \|$ to denote its spectral norm. We use $\| A \|_0$ to denote the number of non-zero entries in $A$. We use $\| A \|_F$ to denote the Frobenius norm of $A$. 
For a square and full rank matrix $A$, we use $A^{-1}$ to denote its inverse. 

\subsection{Basic Definitions}
\label{sec:basic_definitions}
We provide a definition for Boolean semiring.
\begin{definition}
    The Boolean semiring is the set $\{0,1\}$ equipped with addition corresponding to logical OR (i.e. $x + y = 0$ if $x = y = 0$ and $x + y = 1$ otherwise) and multiplication corresponding to logical AND (i.e. $x \cdot y = 1$ if $x = y = 1$ and $x \cdot y = 0$ otherwise).
\end{definition}

We define Bernoulli random variable as follows:
\begin{fact}\label{fact:chernoff_poisson}
    For any $0 < c < 1$ and $0\le p \le \epsilon$, one can estimate the mean of a Bernoulli random variable $\mathrm{Ber}(p)$ to error $c\cdot \epsilon$ with probability $1 - \delta$ using $O(c^{-2}\epsilon^{-1})$ samples.
\end{fact}

\subsection{Discrete probability tools}
\label{sec:tools}

\begin{definition}
    Given a vector $v\in\mathbb{N}^r$, define a \emph{fibre} of a vector to be a set of all indices whose entries are equal to a particular value.
\end{definition}

\begin{lemma}[\cite{fks20}]\label{lem:large_fibre}
    Let $w\in \{0,1\}^r$ be a random vector with exactly $k$ ones where $k<r/2$. Let $q \geq 2$ be an integer and consider a fixed vector $v\in \Z_q^r$ whose largest fibre has size $r-s$. Then, for any $a\in\Z_q$ we have $\Pr[\langle w, v\rangle \equiv a \pmod{q}]\leq P_{s}$ for some $P_s = 2^{-O(sk/r)}$ if $sk = o(r)$, and $P_s = 2^{-\Omega(1)}$ if $sk=\Omega(r)$.
\end{lemma}

\begin{lemma}[\cite{fks20}]\label{lem:small_fibre}
For $x\in \R^r$ whose largest fibre has size $r-s$, and let $w\in \{0,1\}^r$ be a random vector with $k$ ones. Then
\begin{align*}
    \max_{a\in \R} \Pr[\langle x, w\rangle = a] = O\big(\sqrt{r/(sk)}\big).
\end{align*}
\end{lemma}

In the proof of Proposition~\ref{prop:bound_p_lambda}, we needed the following estimate for the binary Krawtchouk polynomial:

\begin{lemma}[\cite{p19}]\label{lem:bound_krawtchouk}
For $k\leq 0.16r$, $\lambda\leq \frac{r}{2}$, we have
\begin{align*}
    |K_k^r(\lambda)| \leq \binom{r}{k}\cdot \left(1-\frac{2k}{r}\right)^\lambda.
\end{align*}
\end{lemma}

\section{Average-Case Algorithm}
\label{app:defer_tensor}
     In Section~\ref{app:defer_fact_gap} we prove that the non-intersection probabilities $\mu_t$ defined in Section~\ref{subsec:formtensor} are well-separated, which we previously stated informally as Fact~\ref{fact:gap_intro}.
     In Section~\ref{app:defer_lem_extractT} we show how to exploit this fact to construct the tensor $\T$ defined in \eqref{eq:Tdef}.
     In Section~\ref{app:defer_lem_jennrich} we give a sufficient condition for the tensor $\T$ to be decomposable.
     
     Section~\ref{app:main_indep} is the main technical component of this paper where we show that under Assumption~\ref{assume:random}, the columns of $\W$ are linearly independent (over $\R$) with high probability so that we can actually apply tensor decomposition. In Section~\ref{subsec:finite_to_real} we observe that one way to show linear independent of the columns of $\W$ over $\R$ is to show linear independence over $\F_2$. In Section~\ref{subsec:average_independent_odd} we handle the case of odd $k$ by using this observation together with a helper lemma from Section~\ref{subsec:average_helper} involving estimates of binary Krawtchouk polynomials. For even $k$, we cannot reduce to showing linear independence over $\F_2$, so in Section~\ref{subsec:average_independent_even} we handle this case using a combination of ideas from \cite{fks20} and a group-theoretic Littlewood-Offord inequality \cite{js17}.
     
     Finally, in Section~\ref{subsec:puttogether_tensor} we put all the ingredients together to complete the proof of Theorem~\ref{thm:avgcase_intro}.

\subsection{Non-Intersection Probabilities \texorpdfstring{$\mu_t$}{mt} Well-Separated}
\label{app:defer_fact_gap}

We begin by showing that the non-intersection probabilities $\mu_t$ defined in \eqref{eq:mudef} are monotonically decreasing and well-separated. For ease of reading, we recall the definition of $\mu_t$ here. Given sets $S_1,\ldots,S_c\subset[r]$, let $\mu(S_1,\ldots,S_c)$ denote the probability that a randomly chosen subset $T$ of $[r]$ of size $k$ does not intersect any of them. Then \begin{equation}
    \mu(S_1,\ldots,S_c) = \binom{r - |S_1\cup\cdots\cup S_c|}{k}\bigg /\binom{r}{k} \triangleq \mu_{|S_1\cup\cdots \cup S_c|}. \label{eq:mudef2}
\end{equation}

\begin{fact}[Formal version of Fact~\ref{fact:gap_intro}]\label{fact:gap}
    There exist absolute constants $C, C' > 0$ for which the following holds. If $r \ge C\cdot k^2$, then for any $0 \le t\le 3k$, we have that $\mu_t \ge \mu_{t+1} + C' k/r$ and $\mu_t \ge 1 - O(tk/r)$.
\end{fact}

\begin{proof}
    For any $0\le t \le r$, note that 
    \begin{align}
        \frac{\binom{r - t}{k}}{\binom{r}{k}} - \frac{\binom{r - t - 1}{k}}{\binom{r}{k}} &= \frac{r-t-1}{r}\cdot \frac{r-t - 2}{r - 1}\cdots \frac{r - t -k + 1}{r - k + 2}\cdot\left(\frac{r - t}{r - k + 1} - \frac{r - t - k}{r - k + 1}\right). \notag\\
        &\ge \left(1- \frac{t + 1}{r - k + 2}\right)^{k - 1} \cdot \frac{k}{r - k + 1} \notag \\
        &\ge \left(1 - \frac{(t+1)(k-1)}{r - k + 2}\right)\cdot \frac{k}{r - k + 1} \ge \Omega(k/r), \notag
    \end{align} where in the last step we used the assumptions that $r \ge \Omega(k^2)$ and $t \le O(k)$.
    For the second part of the claim, we similarly have that 
    \begin{align}
        \frac{\binom{r - t}{k}}{\binom{r}{k}} &= \frac{r-t}{r}\cdot\frac{r - t - 1}{r - 1}\cdots \frac{r - t - k + 1}{r - k + 1} \notag \\
        &\ge \left(1 - \frac{t}{r - k + 1}\right)^k \notag \\
        &\ge 1 - \frac{tk}{r - k + 1} \ge 1 - O(tk/r). \notag
    \end{align}
\end{proof}

\subsection{Constructing a tensor}
\label{app:defer_lem_extractT}

We now use Fact~\ref{fact:gap} to show how to construct the tensor $\T$ defined in \eqref{eq:Tdef}, namely $\T \triangleq \sum^r_{i=1}\W_i^{\otimes 3}$.

\begin{lemma}[Constructing a tensor, formal version of Lemma~\ref{lem:extractT_intro}]\label{lem:extractT}
    If $m \ge \wt{\Omega}\left(rk\log(1/\delta)\right)$, then with probability at least $1 - \delta$ over the randomness of $\M$, there is an algorithm for computing $\T_{a,b,c}$ for any $a,b,c\in[m]$ in time $O(m^{\omega + 1})$.
\end{lemma}

\begin{proof}
    If entries $a,b,c$ correspond to subsets $S_1,S_2,S_3$ of $[m]$, then \begin{equation}
        \T_{a,b,c} = |S_1\cap S_2\cap S_3| = |S_1 \cup S_2\cup S_3| - |S_1\cup S_2| - |S_2\cup S_3| - |S_1\cup S_3| + 3k. \label{eq:PIE}
    \end{equation}
    By Fact~\ref{fact:chernoff_poisson} and the second part of Fact~\ref{fact:gap} applied to $t = |S_i\cup S_j|$ or $t = |S_1\cup S_2\cup S_3|$, we conclude that any $\mu_{|S_i\cup S_j|}$ or $\mu_{|S_1\cup S_2\cup S_3|}$ can be estimated to error $C'k/2r$ with probability $1 - \delta/m^3$ provided that
    \begin{equation*}
    m \ge \Omega\left(\frac{t^2r}{k}\log(m^3/\delta)\right). 
    \end{equation*}
    Note that $t \le 3k$, so this holds by the bound on $m$ in the hypothesis of the lemma. By the first part of Fact~\ref{fact:gap}, provided these quantities can be estimated within the desired accuracy, we can exactly recover every $|S_i\cup S_j|$ as well as $|S_1\cup S_2\cup S_3|$. So by \eqref{eq:PIE} and a union bound over all $(a,b,c)$, with probability at least $1 - \delta$ we can recover every entry of $\T$.
    
    It remains to show that $\T$ can be computed in the claimed time. Note that naively, each entry would require $O(m)$ time to compute, leading to $O(m^4)$ runtime. We now show how to do this more efficiently with fast matrix multiplication. Fix any $a\in[m]$ and consider the $a$-th slice of $\T$. Recall that for every $b,c\in[m]$, we would like to compute the number of columns $\ell\in[m]$ for which $\M_{a,\ell},\M_{b,\ell},\M_{c,\ell}$ are all zero (note that we can compute the other relevant statistics like the number of $\ell$ for which $\M_{a,\ell}$ and $\M_{b,\ell}$ are both zero in total time $O(m^3)$ across all $a,b$ even naively). 
    
    We can first restrict our attention to the set $L_a$ of $\ell$ for which $\M_{a,\ell} = 0$, which can be computed in time $O(m)$. Let $\M_a$ denote the matrix given by restricting $\M$ to the columns indexed by $L_a$ and subtracting every resulting entry from 1. By design, $(\M_a\M_a^{\top})_{b,c}$ is equal to the number of $\ell\in L_a$ for which $\M_{b,\ell} = \M_{c,\ell} = 0$, and the matrix $\M_a\M_a^{\top}$ can be computed in time $O(m^{\omega})$. The claimed runtime follows.
\end{proof}

\subsection{Tensor decomposition}
\label{app:defer_lem_jennrich}

We can invoke standard guarantees for tensor decomposition to decompose $\T$, the key sufficient condition being that the components $\{\W_1,\ldots,\W_r\}$ be linearly independent over $\R$.

\begin{lemma}[Tensor decomposition, formal version of Lemma~\ref{lem:jennrich_intro}]\label{lem:jennrich}
    Given a collection of linearly independent vectors $w_1,\ldots,w_r\in\R^m$, there is an algorithm that takes any tensor $\T = \sum^r_{i=1}w_i^{\otimes 3}$, runs in time $O(m^{\omega})$, and outputs a list of vectors $\wh{w}_1,\ldots,\wh{w}_r$ for which there exists permutation $\pi$ satisfying $\wh{w}_i = w_{\pi(i)}$ for all $i\in[r]$.
\end{lemma}

\begin{proof}
    The algorithm is simply to run Jennrich's algorithm (\cite{har70, lra93}), but we include a proof for completeness. Pick random $v_1,v_2\in \mathbb{S}^{m-1}$  and define $\M_1 \triangleq \T(\Id,\Id,v_1)$ and $\M_2 \triangleq \T(\Id,\Id,v_2)$. If $\W\in\R^{m\times r}$ is the matrix whose columns consist of $w_1,\ldots,w_r$, then we can write $\M_a = \sum^r_{i = 1} \langle w_i,v_a\rangle w_i w_i^{\top} = \W \mathbf{D}_a \W^{\top}$, where $\mathbf{D}_a\triangleq \mathop{\text{diag}}(\langle w_1,v_a\rangle,\ldots, \langle w_r, v_a\rangle)$. As a result, $\M_a\M_b^{+} = \W\mathbf{D}_a\mathbf{D}_b^+ \W^+$. This gives an eigendecomposition of $\M_a\M_b^+$ because the entries of $\mathbf{D}_a\mathbf{D}_b^+$ are distinct almost surely because $\{w_i\}$ are linearly independent. We conclude that the nontrivial eigenvectors of $\M_a\M_b^+$ are precisely the vectors $\{w_i\}$ up to permutation as claimed. Forming $\M_1$ and $\M_2$ takes $O(m^2)$ time, and forming $\M_a\M_b^+$ and computing its eigenvectors takes $O(m^{\omega})$ time.
\end{proof}

\subsection{Linear independence of \texorpdfstring{$\W$}{W}}
\label{app:main_indep}

We now turn to proving the main technical result in this paper, namely that for $\W$ satisfying Assumption~\ref{assume:random}, the columns of $\W$ are linearly independent over $\R$ with high probability as soon as $m$ is near-linear in $r$.

\begin{theorem}[Linear independence of $\W$, formal version of Theorem~\ref{thm:independent_of_W_intro}]\label{thm:independent_of_W}
Let $\W\in \{0,1\}^{m\times r}$ be a random matrix whose rows are i.i.d. random vectors each following a uniform distribution over $\{0,1\}^r$ with exactly $k$ ones. For constant $k\geq 1$ and $m= \Omega(\max(r,(r/k)\log r))$, the $r$ columns of $\W$ are linearly independent in $\R$ with probability at least $1-\frac{1}{\poly(r)}$.
\end{theorem}

\subsubsection{Reduction to Finite Fields}
\label{subsec:finite_to_real}
We treat the cases of odd and even $k$ separately. For the former, we will show that the columns of $\W$ are linearly independent over $\F_2$ with high probability, which by the following is sufficient for linear independence over $\R$:

\begin{lemma}[Reduction from $\F_2$ to $\R$]\label{lem:reduce_r_to_f}
Let $\W\in \{0,1\}^{m\times r}$. If the columns of $\W$ are linearly independent in $\F_2$, then they are also linearly independent in $\R$.
\end{lemma}

\begin{proof}
We prove a contrapositive statement, i.e., if the columns are linearly dependent in $\R$, then they are still dependent in $\F_2$. 

Let $w_1,\dots, w_n$ be the columns of $\W$. If they are linearly dependent in $\R$, then there exists $c_1,\dots,c_r\in \R$ such that $\sum_{i=1}^r c_i w_i = 0$. By Gaussian elimination, it is easy to see that $c_1,\dots,c_r\in \mathbb{Q}$. By multiplying some common factor, we can get $r$ integers $c_1',\dots,c_r'\in \Z$ such that not all of them are even numbers and 
\begin{equation*}
    c_1'w_1 + c_2'w_2 + \cdots + c_r' w_r = 0.
\end{equation*}
Then, apply $(\text{mod}~ 2)$ to both sides of the equation and for any $j\in [m]$, the $j$-th entry of the resulting vector is
\begin{align}\label{eq:mod2}
    \sum_{i=1}^r c_i'\W_{ij} \pmod{2} = \bigoplus_{i=1}^r (c_i' \mod 2)\cdot \W_{ij} = 0,
\end{align}
where the first step follows from $\W_{ij}\in \{0,1\}$. 

Define a vector $a\in \F_2^r$ such that $a_i:=c_i' \mod 2$ for $i\in [r]$. Then, $a\ne 0$ and Eq.~\eqref{eq:mod2} implies that $\W a = 0$ in $\F_2$, which means the columns of $\W$ are linearly dependent in $\F_2$. And the claim hence follows.
\end{proof}

\subsubsection{Linear independence for odd \texorpdfstring{$k$}{k}}
\label{subsec:average_independent_odd}
We are now ready to handle the case of odd $k$. The following lemma proves the $\F_2$ case.

\begin{lemma}[Linear independence in $\F_2$]\label{lem:W_independent}
Give two fixed positive integers $n$ and $r$, for any odd positive integer $k$, if $m=\Omega(\max(r,(r/k)\log r))$, then with probability at least $1-\frac{1}{\poly(r)}$, the columns of $m \times r$ matrix $\W$ are linearly independent in $\F_2$.
\end{lemma}

\begin{proof}
To show that the columns of $\W$ are linearly independent, equivalently, we can show that $\ker(\W)=\emptyset$ with high probability; that is, for all $x\in \F_2^r\backslash\{0\}$, $\W x\ne 0$.

For $1\leq \lambda \leq r$, let $P_{\lambda}:=\Pr[\W u=0]$ for $u\in \F_2^r$ and $|u|=\lambda$. Note that this probability is the same for all weight-$\lambda$ vectors. Then, we have
\begin{align*}
    \Pr[\ker(\W)\ne \emptyset] \leq &~ \sum_{\lambda=1}^r P_\lambda \cdot \Big|\Big\{u\in \F_2^r \Big| |u|=\lambda\Big\}\Big|\\
    \leq &~ \sum_{\lambda=1}^r P_\lambda \cdot \binom{r}{\lambda}.
\end{align*}

Fix $\lambda\in [r]$ and $u\in \F_2^r$ with weight $\lambda$. Since the rows of $\W$ are independent, we have
\begin{equation*}
    P_\lambda = \Pr[\W u = 0] = \prod_{i=1}^m \Pr[\langle w_i, u\rangle = 0] = (\Pr[\langle w, u\rangle=0])^m,
\end{equation*}
where $w$ is a uniformly random vector in the set $\{u\in \F_2^r | |u|=k\}$. It's easy to see that $k$ should be an odd number; otherwise, $\W \mathbf{1} = 0$. 

By Proposition~\ref{prop:bound_p_lambda}, we have
\begin{align*}
    \Pr[\ker{\W}\ne \emptyset] \leq &~ \sum_{\lambda = 1}^{r/2} \left(\frac{1}{2} + \frac{1}{2}\left(1-\frac{2k}{r}\right)^\lambda\right)^m \cdot \binom{r}{\lambda} + \left(\frac{1}{2}\right)^m \cdot \binom{r}{\lambda}\\
    \leq  &~ 2^{r-1-m} + \sum_{\lambda = 1}^{r/2} \left(\frac{1}{2} + \frac{1}{2}\left(1-\frac{2k}{r}\right)^\lambda\right)^m \cdot \binom{r}{\lambda}\tag{$\sum_{i=0}^{r/2}\binom{r}{i} = 2^{r-1}$.}\\
    \leq &~ 2^{r-1-m} + \sum_{\lambda = 1}^{r/2} \left(\frac{1+k\lambda/r}{1+2k\lambda/r}\right)^m\cdot \binom{r}{\lambda}\tag{$(1-\frac{2k}{r})^\lambda\leq \frac{1}{1+(2k\lambda)/r}$.}\\
    \leq &~ 2^{r-1-m} + \sum_{\lambda = 1}^{r/2} \exp\left(-\frac{km\lambda/r}{1+k\lambda/r}\right)\cdot \binom{r}{\lambda}\tag{$1-x\leq e^{-x}$.}\\
    \leq &~ 2^{r-1-m} + \sum_{\lambda = 1}^{r/2} \left(\exp\left(-\frac{km/r}{1+k\lambda/r}+\log (er/\lambda)\right)\right)^\lambda\tag{$\binom{r}{\lambda}\leq (er/\lambda)^\lambda$.}
\end{align*}
Suppose $m>(r/k)\log(er/\lambda)+\lambda \log(er/\lambda)$ for $0<\lambda < r/2$, then we have
\begin{align*}
    -\frac{km/r}{1+k\lambda/r}+\log (er/\lambda) < &~ -\frac{\log(er/\lambda) + k\lambda/r\log(er/\lambda)}{1+k\lambda/r}+\log (er/\lambda)\\
    = &~ -\log(er/\lambda)+\log (er/\lambda)\\
    = &~ 0.
\end{align*}
Note that when $0<\lambda < r/2$, $\lambda \log(er/\lambda)\leq \Omega(r)$. Also, $(r/k)\log(er/\lambda)\leq \Omega((r/k)\log r)$. Hence, if we take $m=\Omega(r+(r/k)\log r)$, we have
\begin{equation*}
    \left(\exp\left(-\frac{km/r}{1+k\lambda/r}+\log (er/\lambda)\right)\right)^\lambda \leq \exp\left(\frac{-\lambda \cdot \Omega(\log (r))}{1+k\lambda / r}\right)\leq ~\frac{1}{\poly(r)}.
\end{equation*}

Therefore,
\begin{align*}
    \Pr[\ker{\W}\ne \emptyset] \leq &~ 2^{r-1-m} + \sum_{\lambda = 1}^{r/2} \left(\exp\left(-\frac{km/r}{1+k\lambda/r}+\log (er/\lambda)\right)\right)^\lambda\\
    \leq &~ 2^{-\Omega((r/k)\log r)} + \sum_{\lambda=1}^{r/2} \frac{1}{\poly(r)}\\
    = &~ \frac{1}{\poly(r)},
\end{align*}
and the lemma is then proved. 
\end{proof} 

Combining Lemma~\ref{lem:reduce_r_to_f} and Lemma~\ref{lem:W_independent}, we immediately have the following corollary:
\begin{corollary}[Linear independence in $\R$ for odd $k$]\label{cor:indep_odd}
For $r\geq 0$, odd constant $k\geq 1$,  when $m=\Omega((r/k)\log r)\geq r$, with probability at least $1-\frac{1}{\poly(r)}$, the columns of matrix $\W$ are linearly independent in $\R$.
\end{corollary}

\subsubsection{Helper lemma}
\label{subsec:average_helper}
In the proof of Lemma~\ref{lem:W_independent} above, we required the following helper lemma:

\begin{proposition}\label{prop:bound_p_lambda}
For $1\leq \lambda \leq \frac{r}{2}$, we have either
\begin{equation*}
    P_\lambda \leq \left(\frac{1}{2} + \frac{1}{2}\left(1-\frac{2k}{r}\right)^\lambda\right)^m  \quad \text{and}~~~P_{r-\lambda} \leq \left(\frac{1}{2}\right)^m,
\end{equation*}
or
\begin{equation*}
    P_\lambda \leq \left(\frac{1}{2}\right)^m \quad \text{and}~~~P_{r-\lambda} \leq \left(\frac{1}{2} + \frac{1}{2}\left(1-\frac{2k}{r}\right)^\lambda\right)^m.
\end{equation*}
\end{proposition}
\begin{proof}
We can write the probability of $\langle w,u\rangle = 0$ as follows:
\begin{equation*}
    \Pr[\langle w,u\rangle=0] = \binom{r}{k}^{-1}\cdot \sum_{i=0,~ i~\text{even}}^{k-1} \binom{\lambda}{i} \binom{r-\lambda}{k-i}.
\end{equation*}

We first consider the following sum with alternating signs:
\begin{equation*}
    K_k^{r}(\lambda) := \sum_{i=0}^k \binom{\lambda}{i}\binom{r-\lambda}{k-i} (-1)^i,
\end{equation*}
which is the binary Krawtchouk polynomial.

Then, for $1\leq \lambda \leq \frac{r}{2}$, we have
\begin{align*}
    \Pr[\langle w, u\rangle = 0] = &~ \binom{r}{k}^{-1}\sum_{i=0,~ i~\text{even}}^{k-1} \binom{\lambda}{i} \binom{r-\lambda}{k-i}\\
    = &~ \frac{1}{2}+\frac{K_k^r(\lambda)}{2\binom{r}{k}}.
\end{align*}
By symmetry, for $\frac{r}{2}<\lambda <r$,
\begin{align*}
    \Pr[\langle w,u\rangle=0] = &~ \binom{r}{k}^{-1} \sum_{i=0,~i~\text{even}}^{k-1}\binom{\lambda}{i}\binom{r-\lambda}{k-i}\\
    = &~ \binom{r}{k}^{-1} \sum_{i=1,~i~\text{odd}}^{k}\binom{r-\lambda}{i}\binom{\lambda}{k-i}\\
    = &~ \frac{1}{2} - \frac{K_k^r(r-\lambda)}{2\binom{r}{k}}.
\end{align*}

By Lemma~\ref{lem:bound_krawtchouk}, we know that for $1\leq \lambda \leq \frac{r}{2}$, $|K_k^r(\lambda)|\leq \binom{r}{k}\left(1-\frac{2k}{r}\right)^\lambda$.

Therefore, if $K_k^r(\lambda)>0$, then
\begin{align*}
    P_\lambda \leq &~ \Pr[\langle w,u\rangle=0]^m \leq \left(\frac{1}{2} + \frac{1}{2}\left(1-\frac{2k}{r}\right)^\lambda\right)^m,\\
    P_{r-\lambda} \leq &~ \left(\frac{1}{m}\right)^m.
\end{align*}
If $K_k^r(\lambda)<0$, then
\begin{align*}
    P_{r-\lambda} \leq &~ \left(\frac{1}{2} + \frac{1}{2}\left(1-\frac{2k}{r}\right)^\lambda\right)^m,\\
    P_{\lambda} \leq &~ \left(\frac{1}{m}\right)^m.
\end{align*}
The proposition hence follows.
\end{proof}

\subsubsection{Linear independence for even \texorpdfstring{$k$}{k}}
\label{subsec:average_independent_even}
Next, we turn to the case of even $k$. When $k$ is even, we cannot use Lemma~\ref{lem:reduce_r_to_f} because the matrix is not linearly independent in $\F_2$. Instead, we use a variant of the proof in \cite{fks20} to show that in this case, the linear independence of the columns of $\W$ still holds with high probability:
\begin{lemma}\label{lem:indep_even}
Give two fixed positive integers $n$ and $r$, for any even positive integer $k$, if $m=\Omega( r + (r/k)\log r)$, then with probability at least $1-\frac{1}{\poly(r)}$, the columns of $m \times r$ matrix $\W$ are linearly independent in $\R$.
\end{lemma}

\begin{proof}
We first show that it suffices to consider the case when $\W \in \mathbb{Z}^{m \times r}$ is an integral matrix. Suppose the columns of $\W$ are not linearly independent, i.e., there exists $x\in \mathbb{Q}^r$ such that $\W x = 0$. Then, we will be able to multiply by an integer and then divide by a power of two to obtain a vector $c\in \Z^r$ with at least one odd entry such that $\langle w_i, c\rangle = 0$ for all $i\in [m]$ where $w_i^\top$ 
denotes the $i$-th row of $\W$. 

To upper bound the column-singular probability, we have
\begin{align*}
    \Pr[\W~\text{is column-singular}]\leq &~ \Pr[\W x=0~\text{for some}~x~\text{with}~|\supp(x)|<r] \\
    &~ +  \Pr[\W x=0~\text{for some}~x~\text{with}~|\supp(x)|=r].
\end{align*}
For the first term, it can be upper bounded by 
\begin{equation*}
    \Pr[\langle w_i, x\rangle \equiv 0\pmod{2}~\forall i\in [r]~\text{for some}~x\ne \mathbf{1}],
\end{equation*}
where $\mathbf{1}$ is an length-$r$ all ones vector. 

Note that the reduction (Lemma~\ref{lem:reduce_r_to_f}) fails only when all of the entries of $c$ are odd because $\W \mathbf{1} = 0$ in $\F_2$ when $k$ is even. Hence, We can use almost the same calculation in Lemma~\ref{lem:W_independent} to upper bound this probability by $r^{-\Omega(1)}$ when $m=\Omega((r/k)\log r)$.

For the second term, let $\mathcal{S}_r$ denote the set of vectors with support size $r$. Define a \textit{fibre} of a vector to be a set of all indices whose entries are equal to a particular value. And define a set $\mathcal{P}$ to be
\begin{equation*}
    \mathcal{P}:=\{c\in \mathbb{Z}^r: c~\text{has largest fibre of size at most }~(1-\delta) r\}
\end{equation*}
Then, we have
\begin{align*}
    \Pr[\W x=0~\text{for some}~x~\text{with}~|\supp(x)|=r]\leq &~ \Pr[\exists x\in \mathcal{S}_r\backslash \mathcal{P}:\langle w_i, x\rangle=0~\forall i\in [m]]\\
    &~ + \Pr[\exists x\in \mathcal{S}_r\cap \mathcal{P}:\langle w_i, x\rangle=0~\forall i\in [m]]\\
    \leq &~ B_1 + B_2.
\end{align*}

For $B_1$, let $q=k-1$. It suffices to prove that there is no non-constant vector $c\in \Z_q^r$ with largest fibre of size at least $(1-\delta) r$ such that $\langle w_i, c\rangle \equiv 0\pmod {q}$ for all $i\in [m]$. Then, by Lemma~\ref{lem:large_fibre}, let $t=O(r/k)$ and the probability can be upper bounded by
\begin{align*}
    B_1\leq &~ \sum_{s=1}^{\delta r} \binom{r}{s} q^{s+1} (P_s)^{m_1}\\
    \leq &~ \sum_{s=1}^{t} 2^{s \log r + (s+1)\log q -O(sk/r)\cdot m } + \sum_{s = t +1}^{\delta r}2^{s \log r + (s+1)\log q - \Omega(m) } \\
    \leq &~ \delta r \cdot 2^{-\Omega(\log r)}\\
    = &~ r^{-O(1)},
\end{align*}
if we take $m_1=\Omega((r/k)\log r)$ and $\delta <1$.

For $B_2$, we need to apply a union bound:
\begin{equation*}
    B_2 \leq |\mathcal{S}_r\cap \mathcal{P}|\cdot \left(\max_{x\in \mathcal{P}}\Pr[\langle w,x\rangle = 0]\right)^m.
\end{equation*}
In fact, it suffices to consider the vectors in $\mathbb{Z}_q^r$, where $q$ is a prime in $[k-1, 2k]$. We need the following group-theoretic Littlewood-Offord inequality:
\begin{lemma}[Corollary 1 in \cite{js17}]\label{lem:littlewood_group}
Let $q\geq 2$ be a prime. Let $x\in \mathbb{Z}_q^n$ with $|\supp(x)|=n$. Let $w_i\in \{-1, 1\}_{\mathbb{Z}_q}$ be a Bernoulli random variable for $i\in [n]$. Then,
\begin{equation*}
    \sup_{g\in \mathbb{Z}_q}\Pr[\langle x, w\rangle = g]\leq 3\max\left\{\frac{1}{q}, \frac{1}{\sqrt{n}}\right\}.
\end{equation*}
\end{lemma}

Then, similar to the proof of Lemma 4.2 in \cite{fks20}, Lemma~\ref{lem:littlewood_group} implies the following claim:
\begin{claim}
Let $0<k\leq r/2$, and consider a vector $x\in \mathbb{Z}_q^r$ ($q$ prime) whose largest fibre has size $r-s$, and let $w\in \{0,1\}^r$ be a random vector with exactly $k$ ones. Then, 
\begin{equation*}
    \sup_{g\in \mathbb{Z}_q}\Pr[\langle x, w\rangle = g] = O\left(\max\{1/q, \sqrt{r/(ks)}\}\right).
\end{equation*}
\end{claim}
In our case, $s=\delta r$ and $\sqrt{r/(ks)}=(\delta k)^{-1/2}$. That is,
\begin{equation*}
    \max_{x\in \mathcal{P}}\Pr[\langle w,x\rangle = 0] \leq O\left(\frac{1}{\sqrt{\delta k}}\right)
\end{equation*}
Then, we apply union bound for $B_2$:
\begin{align*}
    B_2 \leq &~ (q-1)^r \cdot \left(\frac{1}{\sqrt{\delta k}}\right)^m\\
    \leq &~ (2k)^{r}\cdot (\delta k)^{-m/2}\\
    = &~ \exp(O(r\log (k) - m\log(\delta k)))\\
    = &~ \exp(-\Omega(r\log k))\\
    \leq &~ \frac{1}{\poly(r)},
\end{align*}
if we take $\delta$ to be a constant in $(0, 1)$ and $m=\Omega(r)$.

Combining them together, we get that
\begin{equation*}
    \Pr[\W~\text{is column-singular}] \leq \frac{1}{\poly(r)}
\end{equation*}
if $m=\Omega(\max\{r, (r/k)\log r\})$.
\end{proof}

\subsection{Putting Everything Together}
\label{subsec:puttogether_tensor}

In this section we formally show how to conclude our main average-case guarantee:
\begin{theorem}[Average-case guarantee, formal version of Theorem~\ref{thm:avgcase_intro}]\label{thm:avgcase}
    Fix any integer $2\le k \le r$, failure probability $\delta \in (0,1)$, and suppose $m \ge \widetilde{\Omega}(rk\log (1/\delta) )$. Let $\W\in\{0,1\}^{m\times r}$ be generated by the following random process: for every $i\in[m]$, the $i$-th row of $\W$ is a uniformly random $k$-sparse binary vector. Define $\M \triangleq \W\W^{\top}$ where matrix multiplication is over the Boolean semiring. There is an algorithm which runs in $O(m^{\omega + 1})$ time and, with probability $1 - \delta$ over the randomness of $\W$, outputs a matrix $\widehat{ \W } \in \{ 0 , 1 \}^{ m \times r }$ whose columns are a permutation of those of $\W$.\footnote{$\omega \approx 2.373$ is the exponent of matrix multiplication.} 
\end{theorem}
\begin{proof}
    By Lemma~\ref{lem:extractT}, as long as $m$ satisfies the bound in the hypothesis, with probability at least $1 - \delta$ one can successfully form the tensor $\T = \sum^r_{i=1}\W^{\otimes 3}_i$ in time $O(m^{\omega + 1})$. By Theorem~\ref{thm:independent_of_W}, the columns of $\W$ are linearly independent with high probability. By Lemma~\ref{lem:jennrich}, one can therefore recover the columns of $\W$ up to permutation.
\end{proof}


\section{Connections between \texorpdfstring{\Batched{}}{BvK-Sum}, \texorpdfstring{\SSNMF}{SSNMF}, InstaHide}
\label{sec:instahide_defer}

In this section, we begin by formally defining \Batched{} and describing the connection to \SSNMF{} elucidated in \cite{clsz21,carlini_attack}. In Section~\ref{sec:defer_recover_heavy} we then prove Theorem~\ref{thm:instahide_informal} by describing an algorithm that takes a solution to the \SSNMF{} instance corresponding to a \Batched{} instance and extracts more fine-grained information about the latter, specifically certain coordinates of the unknown database generating the \Batched{} instance. As we will explain, the motivation for this particular recovery guarantee comes squarely from designing better attacks on InstaHide.


\subsection{Connection to batched \texorpdfstring{$k$}{k}-vector sum}
\label{sec:batched_k_sum}

We first describe an extension of the well-studied $k$-sum problem \cite{k72,e95,p10,al13,abb19} to a ``batched'' setting. Recall that the classic $k$-sum problem asks: given a collection of $r$ numbers, determine whether there exists a subset of size $k$ summing to zero. One can consider an analogous question where we instead have a collection of $r$ vectors in $\R^d$, and more generally, instead of asking whether some subset of size $k$ sums to the zero vector, we could ask the following search problem which has been studied previously \cite{biwx11,cattaneo2014parameterized,alw14}:
\begin{definition}[$k$-Vector-Sum]
Given a database $\X$ consisting of $d$-dimensional vectors $x_1, \cdots, x_r$, for a fixed vector $y$ and an integer $k$, we promise that there is a set $S \subset [n]$ such that $|S|=k$ and $y = \sum_{i \in S} x_i$. We can observe $y$ and have access to database $\X$, our goal is to recover set $S$.
\end{definition}
Even when $d = 1$, the Exponential Time Hypothesis implies that no algorithm can do better than $r^{o(k)}$, and this essentially remains true even for vectors over finite fields of small characteristic \cite{biwx11}.

We consider two twists on this question. First, instead of a single vector $y$, imagine we got a \emph{batch} of multiple vectors $y_1,\ldots,y_m$, each of which is the sum of some $k$ vectors in the database, and the goal is to figure out the constituent vectors for each $y_i$. Second, given that there might be some redundant information among the $y_i$'s, it is conceivable that under certain assumptions on the database $\X$, we could even hope to solve this problem without knowing $\X$. These considerations motivate the following problem:
\begin{definition}[{\Batched}]\label{def:bksum}
Given \emph{unknown} database $\X$ which is a list of vectors $x_1, \cdots, x_r \in \R^d$ and a fixed integer $k$. For a set of vector $y_1, \cdots, y_m \in \R^d$, for each $j \in [m]$, we promise that there is a set $S_j \subset [n]$ such that $|S_j|=k$ and $y_j = \sum_{i \in S_j} x_i$. Given $y_1, \cdots, y_m$, our goal is to recover sets $S_1, \cdots, S_m$.
\end{definition}

As we discuss in Section~\ref{sec:defer_recover_heavy}, \Batched{} is closely related to existing attacks on a recently proposed scheme for privately training neural networks called InstaHide. 

\subsection{Similarity Oracle}

Taking a step back, note that \Batched{} could even be harder than \ksum{} from a worst-case perspective, e.g. if $y_1 = \cdots = y_m$. The workaround we consider is motivated by the aforementioned applications of \Batched{} to InstaHide \cite{clsz21,carlini_attack}. It turns out that in these applications, the unknown database $\X$ possesses additional properties that allow one to construct the following oracle, e.g. by training an appropriate neural network classifier \cite{carlini_attack}:

\begin{definition}[Similarity oracle]\label{defn:simmatrix_intro}
    Recall the notation of Definition~\ref{def:bksum}. Given vectors $\{y_1,\ldots,y_m\}$ generated by \Batched, let ${\cal O}$ denote the oracle for which ${\cal O}(i,j) = |S_i\cap S_j| \neq \emptyset]$ for all $i,j\in[m]$.
    Also define the \emph{selection matrix} $\W\in\{0,1\}^{m\times r}$ to be the matrix whose $i$-th row is the indicator vector for subset $S_i$.
\end{definition}

The following simple observation tells us that given an instance of \Batched{} together with a similarity oracle, we can immediately reduce to an instance of \SSNMF{}:

\begin{fact}\label{fact:tropicalfactorization}
    If the matrix $\M \in \R^{m \times m}$ has entries $\M_{i,j} = {\cal O}(i,j)$, then it satisfies $\M = \W\W^{\top}$.
\end{fact}

\subsection{An Improved Attack on InstaHide}
\label{sec:defer_recover_heavy}

Given a similarity oracle, we can in fact do more than just recover $S_1,\ldots,S_m$: provided that $d$ is sufficiently large, we can use the matrix $\W$ we have recovered as well as the vectors $y_1,\ldots,y_m$ to solve a collection of linear systems to recover $x_1,\ldots,x_m$. In this section, we show a stronger guarantee: recovery is possible even if we only have access to the \emph{entrywise absolute values} of the $y_i$'s. To motivate this result, we first spell out how exactly InstaHide \cite{huang2020instahide} relates to the \Batched{} problem considered in this section.

\begin{definition}\label{def:instahide}
    From a \emph{private dataset} $\X = \{x_1,\ldots,x_r\}$, InstaHide generates a \emph{synthetic dataset} by sampling $\W\in\{0,1\}^{m\times r}$ according to Assumption~\ref{assume:random} and outputting the vectors $z_1,\ldots,z_m$ given by $z_i = |\W_i \X|$, where $|\cdot|$ denotes entrywise absolute value, $\W_i$ is the $i$-th row of $\W$, and, abusing notation, $\X$ denotes the $r\times d$ matrix whose $j$-th row is $x_j$. In light of Definition~\ref{def:bksum}, let $S_i\subset[r]$ denote the support of the $i$-th row of $\W$.
\end{definition}

Note that the synthetic dataset $\{z_1,\ldots,z_m\}$ in Definition~\ref{def:instahide} is simply given by the entrywise absolute values of the vectors $\{y_1,\ldots,y_m\}$ in Definition~\ref{def:bksum} if the size-$k$ subsets $S_1,\ldots,S_m$ there were chosen uniformly at random. 

It was shown in \cite{huang2020instahide} that by training a neural network on the synthetic dataset generated from a private image dataset, one can still achieve good classification accuracy, and the hope was that by taking entrywise absolute values, one could conceal information about the private dataset. This has since been refuted empirically by \cite{carlini_attack}, and provably by \cite{clsz21,hstzz20} for extremely small values of $k$, but a truly polynomial-time, provable algorithm for recovering private images from synthetic ones generated by InstaHide had remained open, even given a similarity oracle.

Here we close this gap by showing how to efficiently recover most of the private dataset by building on our algorithm for \SSNMF{}. The pseudocode is given in Algorithm~\ref{alg:getheavy}.


\begin{algorithm}\caption{\textsc{GetHeavyCoordinates}($\M$)}\label{alg:getheavy}
	\DontPrintSemicolon
	\LinesNumbered
	\KwIn{Similarity matrix $\M$ for synthetic images generated from unknown private dataset $\X$}
	\KwOut{Matrix $\wh{\X}$ which approximates heavy entries of $\X$ (see Theorem~\ref{thm:instahide_formal})}
	    $\W\gets${\sc TensorRecover}($\M$).\;
	    \For{$j\in[d]$}{
	        Let $z\in\R^m$ have $i$-th entry equal to the $j$-th coordinate of synthetic image $i$.\;
            Form the vector $\wt{p}' \triangleq \frac{1}{m}\sum^m_{i=1}\left(w_i - \frac{k-1}{r-2}\vec{1}\right) \cdot z_i^2$, where $w_i$ is the $i$-th row of $\W$.\;
            Set the $j$-th row of $\wh{\X}$ to be $\wt{p}'\cdot \frac{r(r-1)}{k(r-2k+1)}$.\;
	    }
	    \Return $\wh{\X}$.\;
\end{algorithm}
The following lemma is the main ingredient for analyzing Algorithm~\ref{alg:getheavy}, whose main guarantees were informally stated in Theorem~\ref{thm:instahide_informal}. It essentially says that given $\W$ and the $\ell$-th coordinates of all $z_1,\ldots,z_m$, one can approximately reconstruct the $\ell$-th coordinates of all $x_1,\ldots,x_r$ which are sufficiently ``heavy.'' Roughly, an index $i\in[r]$ is ``heavy'' if the magnitude of the $\ell$-th coordinate of $x_i$ is roughly $k$ times larger than the average value in that coordinate across all $x_1,\ldots,x_r$.

\begin{lemma}\label{lem:recoverheavy}
    For any absolute constant $\eta > 0$, there is an absolute constant $c>0$ for which the following holds as long as $m \ge \Omega(\log(d/\delta))$. There is an algorithm {\sc GetHeavyCoordinates} that takes as input $\W\in\{0,1\}^{m\times r}$ satisfying Assumption~\ref{assume:random} and vector $z\in\R^m$ satisfying $|\W p|= z$ for some vector $p\in\R^r$, runs in time $O(r\cdot m)$, and outputs $\wh{p}$ such that for every $i\in[r]$ for which $|p_i| \ge (ck/r)\cdot\overline{p}$, we have that $\wt{p}_i = p_i \cdot (1\pm \eta)$.
\end{lemma}

We will need the following basic calculation. Henceforth, given a vector $p$, let $\overline{p}$ denote the sum of its entries.

\begin{restatable}{fact}{esp}\label{fact:esp}
    For any vector $p\in\R^r$, \begin{equation}
        \E_S[\langle e_S,p\rangle^2] = \frac{k(r-k)}{r(r-1)}\norm{p}^2_2 + \frac{k(k-1)}{r(r-1)}\overline{p}^2
    \end{equation} where the expectation is over a random size-$k$ subset $S\subset[r]$.
\end{restatable}

\begin{proof}
    Let $\xi_i$ denote the indicator for the event that $i\in S$ so that \begin{align*}
        \E_S[\langle e_S,p\rangle^2] &= \E_S\left[\sum_{i,j\in[r]} \xi_i \xi_j p_i p_j\right] \\
        &= \sum_{i\in[r]} p_i^2\cdot \E[\xi_i] + \sum_{i\neq j} p_ip_j \E[\xi_i\xi_j] \\
        &= \left(\frac{k}{r} - \frac{k(k-1)}{r(r-1)}\right)\norm{p}^2_2 + \frac{k(k-1)}{r(r-1)}\overline{p}^2 \\
        &= \frac{k(r-k)}{r(r-1)}\norm{p}^2_2 + \frac{k(k-1)}{r(r-1)}\overline{p}^2.
    \end{align*} as claimed.
\end{proof}

We are now ready to prove Lemma~\ref{lem:recoverheavy}.

\begin{proof}[Proof of Lemma~\ref{lem:recoverheavy}]
    Define the vector \begin{equation}
        \wt{p} \triangleq \E_S\left[\langle e_S, p\rangle^2 \cdot e_S\right],
    \end{equation} where the expectation is over a random subset $S\subset[r]$ of size $k$, and $e_S\in\{0,1\}^k$ is the indicator vector for the subset $S$. The $i$-th entry of $\wt{p}$ is given by
    \begin{align*}
        \wt{p}_i = &~ \binom{r}{k}^{-1}\cdot \sum_{S\in \binom{r-1}{[k-1]}} (\langle e_S, p\rangle + p_i)^2\\
        = &~ \binom{r}{k}^{-1}\cdot \sum_{S\in \binom{r-1}{[k-1]}} p^2_i + 2p_i\cdot p_S + p^2_S \tag{$p_S \triangleq \sum_{j\in S} p_j$.}\\
        = &~ \binom{r}{k}^{-1} \cdot \left(p^2_i \cdot \binom{r-1}{k-1} + 2p_i \cdot \sum_{S\in \binom{r-1}{[k-1]}} p_S + \sum_{S\in \binom{r-1}{[k-1]}} p^2_S\right)
    \end{align*}
    
    For the second term, we have
    \begin{equation*}
        \sum_{S\in \binom{r-1}{[k-1]}} p_S = \sum_{j\in [r]-\{i\}} p_j \cdot \binom{r-2}{k-2}.
    \end{equation*}
    
    For the third term, we have
    \begin{align*}
        \sum_{S\in \binom{r-1}{[k-1]}} p^2_S = &~  \sum_{S\in \binom{r-1}{[k-1]}} \sum_{j,\ell\in S}p_jp_\ell\\
        = &~ \sum_{S\in \binom{r-1}{[k-1]}} \sum_{j\in S} p_j^2 + \sum_{S\in \binom{r-1}{[k-1]}} \sum_{j\ne \ell} p_j p_\ell\\
        = &~ \sum_{j\in [r]-\{i\}} p_j^2 \cdot \binom{r-2}{k-2} + \sum_{j,\ell\in [r]-\{i\}, j\ne \ell} p_jp_\ell \cdot \binom{r-3}{k-3}.
    \end{align*}
    
    Hence, the $i$-th entry of $\E_S[\langle e_S,p\rangle^2\cdot e_S]$ is
    \begin{align*}
        \Big(\E_S\left[e_S\cdot \langle e_S,p\rangle^2\right]\Big)_i = &~ p^2_i \cdot \frac{k(r-2k+1)}{r(r-1)} + \|p\|_2^2 \cdot \frac{k(k-1)(r-k)}{r(r-1)(r-2)}\\
        + &~ 2p_i \overline{p} \cdot \frac{k(k-1)}{r(r-1)} + \overline{p}^2\cdot \frac{k(k-1)(k-2)}{r(r-1)(r-2)}.
    \end{align*}
    We conclude by Fact~\ref{fact:esp} that the $i$-th entry of $\E_S\left[\langle e_S,p\rangle^2\cdot\left(e_S - \frac{k-1}{r-2}\cdot \vec{1}\right)\right]$ is bounded by \begin{align*}
    \left(\E_S\left[\left(e_S - \frac{k-1}{r-2}\vec{1}\right)\cdot \langle e_S,p\rangle^2\right]\right)_i = &~ p^2_i \cdot \frac{k(r-2k+1)}{r(r-1)} + 2p_i\overline{p}\cdot \frac{k(k-1)}{r(r-1)} \\
    + & ~ \overline{p}^2 \cdot \frac{k(k-1)(2k - 3)}{r(r-1)(r-2)}
    \end{align*}
    
    We do not have exact access to $\wt{p}$, but we may form the unbiased estimator
    \begin{equation}
        \wt{p}' \triangleq  \frac{1}{m} \sum_{i=1}^m \left(w_i - \frac{k-1}{r-2}\vec{1}\right) \cdot z_i^2,
    \end{equation} where $w_i$ is the $i$-th row of $\W$. For any $i\in[m]$, each coordinate of $w_i\cdot z_i^2$ is bounded within the interval $\left[-\norm{z}_{\infty}^2,\norm{z}^2_{\infty}\right]$, so by Chernoff, provided that $m \ge \Omega(\log(d/\delta)/\epsilon^2)$, we ensure that $\wt{p}'_i \in \wt{p}_i(1 \pm \epsilon)$ for all $i$ with probability at least $1 - \delta$.
    Now consider the following estimator for $p^2_i$:
    \begin{equation}
        \wh{q}_i \triangleq \wt{p}'_i\cdot \frac{r(r-1)}{k(r-2k+1)}. \label{eq:qdef}
    \end{equation}
    We can thus upper bound the error of $\wh{q}_i$ relative to $p^2_i$ by
    \begin{align*}
        \frac{\overline{p}}{p_i} \cdot O\left(\frac{k}{r}\right) + \left(\frac{\overline{p}}{p_i}\right)^2 \cdot O\left(\frac{k^2}{r^2}\right) \pm \epsilon\cdot O\left(1 + \frac{\overline{p}}{p_i} \frac{k}{r} + \frac{\overline{p}^2}{p^2_i} \frac{k^2}{r^2}\right).
    \end{align*}
    If we assume that $|p_i| \geq \Omega(k/r)\overline{p}$ and $\epsilon = O(1)$ for appropriately chosen constant factors, then we have that $\wh{q}_i \in p^2_i \cdot (1 \pm \eta)$ as desired.
\end{proof}


We are now ready to prove the main guarantee for our attack on InstaHide, originally stated informally in Theorem~\ref{thm:instahide_formal}.

\begin{theorem}[Formal version of Theorem~\ref{thm:instahide_informal}]\label{thm:instahide_formal}
    For any absolute constant $\eta > 0$, there is an absolute constant $c > 0$ for which the following holds. Fix any integer $k \ge 2$, failure probability $\delta \in (0,1)$, and suppose $m \ge \widetilde{\Omega}(rk\log(d/\delta))$. Given a synthetic dataset of size $m$ generated by {\Batched} from a matrix $\X$, together with its similarity oracle, there is an $O(m^{\omega + 1} + d\cdot r\cdot m)$-time algorithm which outputs a matrix $\wh{\X}$ such that for any $(i,j)\in[r]\times[d]$ satisfying $|\X_{i,j}| \ge (ck/r)\sum_{i'\in[r]}|\X_{i',j}|$, we have that $|\wh{\X}_{i,j}| = |\X_{i,j}|\cdot (1 \pm \eta)$.
\end{theorem}
\begin{proof}
    By Theorem~\ref{thm:avgcase} and the assumed lower bound on $m$, we can exactly recover the selection matrix $\W$ (up to some column permutation) in time $O(m^{\omega + 1})$. Using Lemma~\ref{lem:recoverheavy}, for every pixel index $j\in[d]$ we can run {\sc GetHeavyCoordinates}($\M$) to recover the pixels in position $j$ which are heaviest among the $r$ private images in time $O(m\cdot r)$, yielding the desired guarantee.
\end{proof}



\bibliography{ref}
\ifdefined\isarxivversion
\bibliographystyle{alpha}
\else
\bibliographystyle{plainurl}
\fi

\appendix

\section{Worst-Case Algorithm}
\label{app:defer_csp}

In this section, we give a worst-case quasi-polynomial algorithm for sparse boolean matrix factorization problem. It turns out that our techniques here can handle both \SSNMF{} as well as an asymmetric variant, In Section~\ref{sec:csp_defs} we define this variant and give some background on constraint satisfaction problems (CSPs). Section~\ref{sec:csp_results} gives the algorithm for the asymmetric and symmetric version of the problem by exhibiting a reduction to 2-CSP. Section~\ref{sec:csp_corollary} extends the algorithm to the Boolean semiring.
\subsection{CSP Preliminaries}
\label{sec:csp_defs}
We first define a general (asymmetric) version of \SSNMF{} as follows:
\begin{definition}[Sparse Boolean matrix factorization (Sparse {\BMF})]
Given an $m \times m$ matrix $\mathbf{M}$ where each entry is in $\{0,1,\cdots,k\}$. Suppose matrix $\mathbf{M}$ can be factorized into two matrices $\mathbf{U}\in \{0,1\}^{m \times r}$ and $\mathbf{V} \in \{0,1\}^{r \times m}$, where each row of $\mathbf{U}$ is $k$-sparse and each column of $\mathbf{V}$ is $k$-sparse. 

The task is to find a row $k$-sparse matrix $\wh{\mathbf{U}}\in \{0,1\}^{m\times r}$ and a column $k$-sparse matrix $\wh{\mathbf{V}}\in \{0,1\}^{r\times m}$ such that $\mathbf{M}=\wh{\mathbf{U}}\wh{\mathbf{V}}$.
\end{definition}

We can also define the sparse Boolean matrix factorization as an optimization problem.
\begin{definition}[Sparse BMF, optimization version]
Given an $m\times m$ matrix $\mathbf{M}$ where each entry is in $\{0,1,\cdots,k\}$. The goal is to find a row $k$-sparse matrix $\wh{\mathbf{U}}\in \{0,1\}^{m\times r}$ and a column $k$-sparse matrix $\wh{\mathbf{V}}\in \{0,1\}^{r\times m}$ such that the number of different entries $\|\M - \wh{\U}\wh{\V}\|_0$ is minimized.
\end{definition}

We now recall the definition of 2-CSPs:

\begin{definition}[Max 2-CSP]
A 2-CSP problem is defined by the tuple $(\Sigma, V, E, \mathcal{C})$. $\Sigma$ is an alphabet set of size $q$, $V$ is a variable set of size $n$, $E \subseteq V \times V$ is the constraint set. $V$ and $E$ define an underlying graph of the 2-CSP instance, and $\mathcal{C} = \{C_e\}_{e\in E}$ describes the constraints. For each $e \in E$, $C_e$ is a function $\Sigma \times \Sigma \to \{0, 1\}$. The goal is to find an assignment $\sigma ~:~ V \to \Sigma$ with maximal \emph{value}, defined to be the number of satisfied edges $e = (u, v) \in E$ (i.e., for which $C_e(\sigma(u), \sigma(v)) = 1$).
\end{definition}

We will use the following known algorithm for solving ``dense'' 2-CSP instances:
\begin{theorem}[\cite{dm18}]\label{thm:csp_solver}
Define the density $\delta$ of a 2-CSP instance to be $\delta\triangleq |E|/\binom{|V|}{2}$. For any $0<\epsilon \leq 1$, there is an approximation algorithm that, given any $\delta$-dense 2-CSP instance with optimal value $\mathrm{OPT}$, runs in time $(nq)^{O(\epsilon^{-1}\cdot\delta^{-1}\cdot \log q)}$ and outputs an assignment with value $\mathrm{OPT}-\epsilon|E|$, where $n=|V|$ and $q= |\Sigma|$.
\end{theorem}

\subsection{From Factorization to CSPs}
\label{sec:csp_results}
We give a reduction that can reduce the general sparse BMF problem to a Max 2-CSP problem, and then use a quasi-polynomial time 2-CSP solver to find an approximation solution.
\begin{theorem}[QPTAS for asymmetric sparse BMF, formal version of Theorem~\ref{thm:csp_informal}]\label{thm:qptas_bmf}
Given $m,k,r\geq 0$ and an $m \times m$ matrix $\mathbf{M}$ as the input of an instance of sparse Boolean matrix factorization problem. Let $\mathrm{OPT}$ be the optimal value of the problem, i.e., $\mathrm{OPT} := \min_{\U, \V} \|\M - \U\V\|_0$,
where $\U,\V$ satisfy the sparsity constraints of the problem.

For any $1\geq \epsilon > 0$, there exists an algorithm that runs in
\begin{align*}
m^{O(\epsilon^{-1}k\log r)}r^{O(\epsilon^{-1}k^2\log r)}
\end{align*}
time and finds a row $k$-sparse matrix $\wh{\mathbf{U}}$ and a column $k$-sparse matrix $\wh{\mathbf{V}}$ satisfying
\begin{align*}
\|\mathbf{M}-\wh{\mathbf{U}}\wh{\mathbf{V}}\|_0\leq \mathrm{OPT}+\epsilon m^2.
\end{align*}
\end{theorem}

\begin{proof}
For the input matrix $\mathbf{M}$, let $\U$ and $\V$ be the ground-truth of the factorization. Let $(b_1^\top,\dots,b_m^\top)$ be the rows of $\U$ and $(c_1,\dots,c_m)$ be the columns of $\V$. We construct a 2-CSP instance $\mathcal{F}_A$ that finds $\U$ and $\V$ as follows:
\begin{itemize}
    \item Let $\Sigma=\big\{(q_1,\dots,q_r)~\big|~ q_i\in \{0,1\}~\forall i\in [r] ~\text{and}~\sum_{i\in [r]}q_i=k\big\}$ be the alphabet.
    \item The underlying graph is a bipartite graph. The left-side vertices $V_L=[m]$ corresponding to the rows of $\U$. The right-side vertices $V_R=[m]$ corresponding to the columns of $\V$.
    \item For $e=(u,v)\in V_L\times V_R$, define the constraint $C_e$ to be: for all $(p_1,\dots,p_r), (q_1,\dots,q_r)\in \Sigma\times \Sigma$,
    \begin{align*}
        C_e((p_1,\dots,p_r), (q_1,\dots,q_r))=1 ~\Longleftrightarrow~ \sum_{i=1}^r p_iq_i = A_{u,v}.
    \end{align*}
\end{itemize}

Note that $\mathcal{F}_A$ has value $\mathrm{OPT}$. We can create an assignment from the ground-truth such that $\sigma(u)=b_u$ for $u\in V_L$ and $\sigma(v)=c_v$ for $v\in V_R$. By definition of the sparse Boolean matrix factorization problem, this is a legal assignment. Also, since the number of $A_{u,v}=\langle b_u, c_v\rangle$ for all $(u,v)\in [m]\times [m]$ is $m^2-\mathrm{OPT}$, we can see that all such edges are satisfied by this assignment.

Then, we can run the QPTAS algorithm (Theorem~\ref{thm:csp_solver}) on $\mathcal{F}_A$ and obtain an assignment that at most $\mathrm{OPT}-\epsilon |E|$ constraints are unsatisfied, which means the number of different entries between $\M$ and $\wh{\U}\wh{\V}$ is at most $\mathrm{OPT}-\epsilon m^2$. 

The alphabet size of $\mathcal{F}_A$ is $\binom{r}{k}$. The reduction time is $O(m^2r^k)$ and the 2-CSP solving time is $(mr^k)^{O(\epsilon^{-1}\log (r^k))}$ by Theorem~\ref{thm:csp_solver} since the density of a complete bipartite graph is $\delta=\frac{1}{2}$. The theorem is then proved.
\end{proof}
\begin{remark}\label{remark:compare_kumar}
We briefly compare this to the guarantee of \cite{kumar2019faster}, who obtained a \begin{align*}
2^{O(r^2\log r)}\poly(m)
\end{align*}
constant-factor approximation algorithm. By introducing a sparsity constraint on the rows of $\U,\V$ through our new parameter $k$, we circumvent the exponential dependence on $r$, at the cost of running in time quasipolynomial in $m$. In particular, our guarantee dominates when the rank parameter $r$ is at least roughly $\widetilde{\Omega}(\sqrt{\log m})$, though strictly speaking our guarantee is incomparable because we aim for an additive approximation and only measure error in $L_0$ rather than Frobenius norm.
\end{remark}
A similar reduction can be used to prove a worst-case guarantee for \SSNMF{}, stated informally in Theorem~\ref{thm:csp_informal}.

\begin{theorem}[Formal version of Theorem~\ref{thm:csp_informal}]
\label{thm:qptas_bmf_sym_app}
Given $m,k,r\geq 0$ and a symmetric $m \times m$ matrix $\mathbf{M}$ as the input of a (worst-case) instance of \SSNMF{}. Let $\mathrm{OPT}$ be the optimal value of the problem, i.e., $\mathrm{OPT} := \min_{\W} \|\M - \W\W^\top\|_0$,
where $\W$ is a row $k$-sparse matrix in $\{0,1\}^{m\times r}$. For any accuracy $ \epsilon \in  (0,1) $, there is an algorithm running in time 
\begin{align*}
    m^{O(\epsilon^{-1}k\log r)}r^{O(\epsilon^{-1}k^2\log r)}
\end{align*}
which finds a row $k$-sparse matrix $\wh{\mathbf{W}}$ satisfying
\begin{align*}
    \|\mathbf{M}-\wh{\mathbf{W}}\wh{\mathbf{W}}^\top\|_0\leq \mathrm{OPT}+\epsilon m^2.
\end{align*}
\end{theorem}
\begin{proof}
The construction of the 2-CSP instance $\mathcal{F}_A$ is almost the same as in the proof of Theorem~\ref{thm:qptas_bmf}, except that in this case, the underlying graph is a complete graph, where the vertices $V=[m]$ correspond to the rows of $\W$. Then, each constraint $C_{(u,v)}$ checks whether $\langle b_u, b_v\rangle = \M_{u,v}$. The correctness of the reduction follows exactly the proof of Theorem~\ref{thm:qptas_bmf} and we omit it here. The density of $\mathcal{F}_A$ in this case is $1$, and hence the running time of the algorithm is $(mr^k)^{O(\epsilon^{-1}\log (r^k))}$.
\end{proof}

\subsection{Extension to the Boolean Semiring}
\label{sec:csp_corollary}
A direct corollary of Theorem~\ref{thm:qptas_bmf_sym_app} is that, the sparse BMF over the Boolean semiring can also be solved in quasi-polynomial time.

\begin{corollary}\label{cor:semiring}
Given $m,k,r\geq 0$ and a symmetric Boolean $m \times m$ matrix $\mathbf{M}$. Let $\mathrm{OPT}$ be the optimal value of the problem, i.e., $\mathrm{OPT} := \min_{\W} \|\M - \W\W^\top\|_0$,
where $\W$ is a row $k$-sparse matrix in $\{0,1\}^{m\times r}$ and the matrix multiplication is over the Boolean semiring, i.e., $a+b$ is $a\vee b$ and $a \cdot b$ is $a \wedge b$.  
For any accuracy parameter $ \epsilon \in  (0,1) $, there exists an algorithm that runs in 
\begin{align*}
m^{O(\epsilon^{-1}k\log r)}r^{O(\epsilon^{-1}k^2\log r)}
\end{align*}
time and finds a row $k$-sparse matrix $\wh{\mathbf{W}}$ satisfying 
\begin{align*}
\|\mathbf{M}-\wh{\mathbf{W}}\wh{\mathbf{W}}^\top\|_0\leq \mathrm{OPT}+\epsilon m^2.
\end{align*}
\end{corollary}

\begin{proof}
The construction can be easily adapted to the case when matrix multiplication is over the Boolean semiring, where $a+b$ becomes $a\vee b$ and $a \cdot b$ becomes $a \wedge b$. We can just modify the constraints of the 2-CSP instance in the reduction and it is easy to see that the algorithm still works.
\end{proof}

\begin{remark}\label{rmk:compare_bool}
Factorizing Boolean matrices with Boolean arithmetic is equivalent to the bipartite clique cover problem. It was proved by \cite{chandran2017parameterized} that the time complexity lower bound for the exact version of this problem is $2^{2^{\Omega(r)}}$. Since the approximation error is $\epsilon m^2$, when $\epsilon < \frac{1}{m^2}$, the output of our algorithm is the exact solution. Further, if we do not have the row sparsity condition, i.e., $k=r$, then the time complexity becomes
\begin{align*}
    2^{O(m^2(\log m) \cdot r^2\log r)} .
\end{align*}
In the realm of parameterized complexity (see for example \cite{cfk15}), due to the kernelization in \cite{fmps09}, we may assume $m \leq 2^{r}$ and the running time of our algorithm is $2^{\wt{O}(2^{2r}\cdot r^3)}$, which matches the lower bound of this problem.
\end{remark}

\end{document}